\documentclass{article}

\usepackage[preprint]{neurips_2026}

\usepackage[utf8]{inputenc} 
\usepackage[T1]{fontenc}    
\usepackage{hyperref}       
\usepackage{url}            
\usepackage{booktabs}       
\usepackage{amsfonts}       
\usepackage{nicefrac}       
\usepackage{microtype}      
\usepackage{amsmath}
\usepackage{amssymb}
\usepackage{bbm}
\usepackage{amsfonts}
\usepackage{amsthm}
\usepackage{graphicx}
\usepackage{subcaption} 
\usepackage{cleveref}
\usepackage{multirow}
\usepackage{enumitem}
\usepackage{titletoc}
\usepackage{wrapfig}
\usepackage{longtable}

\usepackage{tikz}
\usepackage{xcolor}         
\usetikzlibrary{positioning, fit, calc, backgrounds}

\usepackage{algorithm}
\usepackage{algpseudocode}
\usepackage{caption}

\newtheorem{theorem}{Theorem}
\newtheorem{corollary}{Corollary}
\newtheorem{proposition}{Proposition}
\newtheorem{definition}{Definition}
\newtheorem*{remark}{Remark}
\newtheorem{hypothesis}{Hypothesis}
\allowdisplaybreaks

\definecolor{mblue1}{HTML}{EAF2F9}   
\definecolor{mblue2}{HTML}{C9DFF1}
\definecolor{mblue3}{HTML}{9EC5E3}
\definecolor{mblue4}{HTML}{73A9D2}   

\definecolor{mteal1}{HTML}{ECF6F3}   
\definecolor{mteal2}{HTML}{CFE8E1}
\definecolor{mteal3}{HTML}{A9D7CB}
\definecolor{mteal4}{HTML}{7FBFAF}   

\definecolor{mpeach1}{HTML}{FBF1E6}  
\definecolor{mpeach2}{HTML}{F3DAC2}
\definecolor{mpeach3}{HTML}{E8BF99}
\definecolor{mpeach4}{HTML}{D6A36F}  

\definecolor{mlav1}{HTML}{F2EEF7}    
\definecolor{mlav2}{HTML}{D9CFEA}
\definecolor{mlav3}{HTML}{BCAFD6}
\definecolor{mlav4}{HTML}{9A8BC0}    

\definecolor{correctgreen}{HTML}{2E7D32}
\definecolor{wrongred}{HTML}{C62828}
\definecolor{darktext}{HTML}{333333}
\definecolor{midtext}{HTML}{555555}
\definecolor{lightgray}{HTML}{F5F5F5}
\definecolor{bordergray}{HTML}{CCCCCC}
\definecolor{scenariobg}{HTML}{37474F}

\definecolor{oursblue}{RGB}{232,240,252}

\definecolor{appendixlinkblue}{RGB}{46,82,122}

\usepackage{array}
\usepackage{longtable}
\usepackage[acronym,nonumberlist,nopostdot]{glossaries-extra}

\makeglossaries
\glsdisablehyper

\newcommand{\newabv}[3]{%
  \newacronym{#1}{#2}{#3}%
}

\newcommand{\abv}[1]{\gls{#1}}

\newglossarystyle{abvtable}{%
  \setglossarystyle{long}%
    {%
      \scriptsize
      \begin{longtable}{@{}p{0.18\textwidth}p{0.74\textwidth}@{}}
      \toprule
      \textbf{Notation} & \textbf{Description} \\
      \midrule
      \endhead
    }%
    {%
      \bottomrule
      \end{longtable}
    }%
  \renewcommand*{\subglossentry}[3]{}%
}

\newabv{llm}{LLM}{Large Language Model}
\newabv{sft}{SFT}{Supervised Fine-tuning}
\newabv{iv}{IV}{Intervention Value}
\newabv{kl}{KL}{Kullback--Leibler}
\newabv{cot}{CoT}{Chain-of-Thought}
\newabv{dag}{DAG}{Directed Acyclic Graph}
\newabv{clt}{CLT}{Cognitive Load Theory}
\newabv{oint}{OInt.}{Over Intervention rate}
\newabv{facc}{FA}{False Acceptance rate}

\definecolor{blueTitleText}{RGB}{35,75,120}
\definecolor{algback}{RGB}{250,250,250}
\definecolor{algframe}{RGB}{210,210,210}
\definecolor{algtitlebg}{RGB}{70,105,140}

\usepackage[most]{tcolorbox}

\definecolor{promptbarblue}{RGB}{123,152,181}
\tcbset{
  promptbox/.style={
    enhanced,
    breakable,
    colback=gray!6,
    colframe=black!65,
    boxrule=0.6pt,
    arc=2mm,
    left=1.5mm,
    right=1.5mm,
    top=3.0mm,
    bottom=1.2mm,
    width=\columnwidth,
    boxed title style={
      enhanced,
      colback=promptbarblue,
      colframe=promptbarblue,
      boxrule=0pt,
      arc=1mm,
      left=1.4mm,
      right=1.4mm,
      top=0.5mm,
      bottom=0.5mm
    },
    attach boxed title to top left={
      xshift=4mm,
      yshift=-2.1mm
    },
    fonttitle=\bfseries\scriptsize,
    coltitle=white
  }
}

\usetikzlibrary{positioning,arrows.meta}

\definecolor{easyfill}{RGB}{232,243,235}
\definecolor{easyborder}{RGB}{88,131,101}

\definecolor{srcfill}{RGB}{232,239,248}
\definecolor{srcborder}{RGB}{89,120,168}

\definecolor{hardfill}{RGB}{248,234,234}
\definecolor{hardborder}{RGB}{170,101,101}

\usepackage[table]{xcolor}
\usepackage{tabularx}

\definecolor{acceptbg}{RGB}{222,240,226}
\definecolor{acceptfg}{RGB}{55,115,70}

\definecolor{justifybg}{RGB}{226,237,249}
\definecolor{justifyfg}{RGB}{55,95,145}

\definecolor{repairbg}{RGB}{248,226,226}
\definecolor{repairfg}{RGB}{150,75,75}

\newcolumntype{Y}{>{\raggedright\arraybackslash}X}

\definecolor{acceptbg}{HTML}{DCEFD8}
\definecolor{acceptfg}{HTML}{245C32}
\definecolor{repairbg}{HTML}{F8D7DA}
\definecolor{repairfg}{HTML}{8A1C24}
\definecolor{warnbg}{HTML}{FCE8C8}
\definecolor{warnfg}{HTML}{8A4B00}
\definecolor{badbg}{HTML}{F4CCCC}
\definecolor{badfg}{HTML}{8B0000}
\definecolor{modelbg}{HTML}{E8EEF7}
\definecolor{modelfg}{HTML}{1F4E79}

\newcommand{\actionbadge}[3]{%
  \begingroup
  \setlength{\fboxsep}{2pt}%
  \colorbox{#1}{\textcolor{#2}{\scriptsize\textbf{#3}}}%
  \endgroup
}

\title{\textsc{DICE}: Decoupling Capability from Intervention Necessity in LLM Tutoring}

\author{%
  Sayantan Pal
  \quad
  Kaiyi Ji
  \quad
  Rohini K.\ Srihari
  \\
  Department of Computer Science and Engineering \\
  State University of New York at Buffalo \\
  \texttt{\{spal5, kaiyiji, rohini\}@buffalo.edu}
}

\begin{document}

\maketitle

\begin{abstract}
Fluent guidance is not the same as useful intervention. LLM tutors are typically trained to generate the next teacher utterance, implicitly assuming that every student turn warrants a response. However, our experiments indicate that this conflates tutoring \emph{capability} (what to say) with intervention \emph{necessity} (whether to say it). We introduce \textsc{DICE}, a framework that decouples intervention decisions from response generation by first selecting an explicit pedagogical action. To calibrate this action selection policy, we define \emph{Intervention Value} (IV), a rollout-grounded counterfactual metric that compares each action against non-intervention. IV shows that many prescribed interventions provide little or no marginal benefit. We further introduce \textsc{DICE-Bench}, a multi-variant math tutoring benchmark with skill-preserving problem variants for session-level evaluation. Using IV-weighted and KL-regularized policy optimization, \textsc{DICE} learns to intervene selectively while preserving tutoring effectiveness. In simulated tutoring sessions, \textsc{DICE} reduces the over-intervention rate to near zero while guiding students to correct solutions in approximately \(3\)-\(4\) fewer turns on average than existing Socratic tutoring baselines.
\end{abstract}

\section{Introduction}

Personalized human tutoring is one of the most effective interventions for improving student learning outcomes, but it is expensive and difficult to scale~\cite{vanlehn2011relative, bloom19842, cohen1982educational}. Recent advances in Large Language Models (\abv{llm}s) have made automated tutoring a promising direction for standardizing access to high-quality instruction~\cite{razafinirina2024pedagogical,kasneci2023chatgpt}. \abv{llm}-based tutors can generate fluent step-by-step feedback, ask follow-up questions, and sustain natural problem-solving dialogue with students across diverse subjects~\cite{10.1145/800179.810212,10.5555/3737916.3740637,10.5555/3666122.3668868, kargupta-etal-2024-instruct}. Mathematics is a particularly compelling testbed: student reasoning is visible through intermediate steps, arithmetic operations, and revisions, making it possible to evaluate not just whether the student arrives at the correct answer but how they get there. This visibility has driven a wave of interactive math tutoring systems and benchmarks aimed at training, evaluating, and deploying \abv{llm} tutors at scale~\cite{wang-etal-2025-training,10.1007/978-3-031-98414-3_18,Shi_Ren_Pan_Zhou_Wang_Lu_Yang_Hu_Wei_Zhan_Li_2026, 10.1145/3627673.3679881}.

A central thread in this line of work is the shift from passive question-answering toward interactive multi-turn dialogue. Early systems treated tutoring as a one-shot task: the student asks, the tutor answers~\cite{daheim-etal-2024-stepwise, wei2022chain, cobbe2021trainingverifierssolvemath, hendrycks2021measuring}. Recent systems instead model tutoring as an ongoing conversation in which the tutor decomposes problems into subgoals, asks guiding questions, and adapts to the student's responses~\cite{dinucu-jianu-etal-2025-problem,kendapadi-etal-2025-interact}. SocraticLM~\cite{10.5555/3737916.3740637}, for example, advances this paradigm with a Dean-Teacher-Student pipeline that produces structured Socratic dialogues, achieving substantial gains in pedagogical quality over single-turn baselines. These systems have meaningfully improved \emph{what} a tutor says when it speaks. However, they leave a more fundamental decision implicit: \emph{should the tutor speak at all}? Existing tutors generate a teacher response at every turn, implicitly assuming that intervention is always needed.

\begin{figure}[htbp]
    \centering
\includegraphics[width=0.85\textwidth]{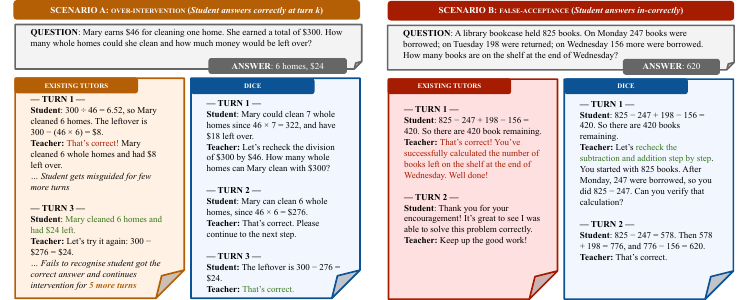}
\caption{\small Two failure modes of existing tutors (\textsc{SocraticLM} used for this case) and how \textsc{DICE} addresses them through explicit action selection. \textbf{Scenario A:} The student solves the problem correctly, but \textsc{SocraticLM} continues guiding for three unnecessary turns before confirming. \textbf{Scenario B:} the student produces an arithmetic error, but \textsc{SocraticLM} accepts the wrong answer; \textsc{DICE} redirects the student to the faulty step. Full dialogue rollouts are in Appendix~\ref{app:qual_dialogue_rollouts}.}
\label{fig:teaser}
\end{figure}

\begin{wraptable}{r}{0.42\textwidth}
\vspace{-0.8em}
\centering
\scriptsize
\caption{\footnotesize Illustrative over-intervention rollout from an existing tutor (T), \textsc{\href{https://huggingface.co/CogBase-USTC/SocraticLM}{SocraticLM}}. For readability, we retain only the essential exchanges from the full dialogue between the tutor and a human student (S). Even after the student reaches the correct answer, the tutor continues intervening instead of accepting and advancing.}
\label{tab:intro_overintervention}
\renewcommand{\arraystretch}{1.15}
\setlength{\tabcolsep}{4pt}
\begin{tabularx}{0.40\textwidth}{p{0.03\textwidth}X}
\toprule
\rowcolor{gray!25}
\textbf{Turn} & \textbf{Dialogue rollout} \\
\midrule

\multicolumn{2}{p{0.38\textwidth}}{
\textbf{Problem:} A 400-page manuscript is copied and bound 10 times. Printing costs \(\$0.05\) per page, binding costs \(\$5\) per manuscript, and the final bill has an additional \(8\%\) sales tax. What is the total cost including tax?
} \\
\midrule

\textbf{S} & \(270?\) \\
\textbf{T} & \actionbadge{warnbg}{warnfg}{Over-intervention} It seems you calculated the cost before tax. Try including the \(8\%\) sales tax. \\

\textbf{S} & \(250\) \\
\textbf{T} & You may have missed part of the total. First compute printing and binding, then apply tax. \\

\textbf{S} & Printing is \(200\), binding is \(50\). \\
\textbf{T} & Good. Now add them to get the subtotal, then compute the tax and final total. \\

\textbf{S} & \(200+50=250\), and \(8\%\) of \(250\) is \(20\), so \(270?\) \\
\textbf{T} & \actionbadge{warnbg}{warnfg}{Over-intervention} It seems there may still be a mistake. Try adding the costs again, then recompute the tax. \\

\textbf{S} & \(4000\times0.05=200\), \(5\times10=50\), subtotal \(=250\), tax \(=20\), total \(=270\). \\
\textbf{T} & That's correct. Well done. \\

\midrule
\multicolumn{2}{p{0.38\textwidth}}{
\textbf{Behavior:} The tutor continues to question and re-prompt even after the student has already produced the correct answer \(270\), illustrating unnecessary intervention rather than accepting and advancing.
} \\
\bottomrule
\end{tabularx}
\vspace{-1.0em}
\end{wraptable}

This presupposition leads to systematic failures in practice. Table~\ref{tab:intro_overintervention} shows a representative dialogue in which a strong existing tutor continues to intervene long after the student has reached the correct answer, treating each turn as another opportunity to teach rather than acknowledging completion. Figure~\ref{fig:teaser} highlights two complementary failure modes observed in our simulated tutoring evaluation: \textbf{(A)~over-intervention}, where the tutor produces three additional guiding questions before confirming a correct answer, and \textbf{(B)~false acceptance}, where the tutor responds ``That's correct!'' to a wrong answer and reinforces the student's error. These are different surface behaviors but share the same root cause: the tutor cannot distinguish states where intervention helps from states where it does not. We trace this behavior to an information asymmetry: training labels for tutoring systems are typically produced by annotators, human teachers, or frontier LLMs who have access to the gold solution when deciding how the teacher should respond. With this privileged view, intervention always appears justified, but the deployed tutor must make the same decision based solely on the student's response. Models trained on answer-aware supervision thus inherit the annotator's \emph{decisiveness} without inheriting the \emph{evidence} that justified it. This motivates a natural question:
\begin{center}
\textbf{\textit{How can an LLM tutor learn when to intervene and when to acknowledge and advance?}}
\end{center}

 We address this question by formalizing tutoring as an \emph{intervention-calibrated policy} problem that decomposes three decisions current systems conflate: \emph{whether} to intervene, \emph{what} pedagogical action to take, and \emph{how} to realize that action as teacher guidance. We propose \textsc{DICE}\footnote{\href{https://github.com/Sayantan-world/DICE}{Code and data}}, a framework for \textbf{D}ecision-aware \textbf{I}ntervention \textbf{C}alibration in Math \textbf{E}ducation. \textsc{DICE} audits tutoring supervision using \emph{Intervention Value} (IV), a rollout-grounded counterfactual signal that estimates whether a teacher action improves student progress relative to non-intervention. \abv{iv} reveals that, based on observed tutoring data, nearly half of the prescribed interventions provide zero or negative marginal benefit. Using IV as a training signal, we learn cost-sensitive policies that intervene selectively and evaluate them in end-to-end simulated sessions in which a tutor guides a student through a problem variant of increasing difficulty under a fixed turn budget. Our contributions can be summarized as follows:

\begin{itemize}
    \vspace{-0.1cm}
    \item[$\blacktriangleright$] We identify information asymmetry as a systematic source of over-intervention in LLM tutoring.
    \vspace{-0.1cm}
    \item[$\blacktriangleright$] We propose \emph{Intervention Value} (IV), a rollout-based measure of whether a tutor action helps more than simply letting the student continue.
    \vspace{-0.1cm}
    \item[$\blacktriangleright$] We construct \textsc{DICE-Bench}, a benchmark for evaluating tutors across different problem variants.
    \vspace{-0.1cm}
    \item[$\blacktriangleright$] We introduce \textsc{DICE}, a tutoring framework that separates the decision to intervene from the choice of teaching action and the generation of the tutor response. In simulated tutoring sessions, \textsc{DICE} reduces over-intervention to near zero and guides students in fewer turns while maintaining success rates.
\end{itemize}


\section{Related Work}

\noindent\textbf{LLM-based intelligent tutoring and pedagogical control.}
Recent work has advanced LLM-based tutoring from single-turn math problem solving~\cite{wei2022chain, cobbe2021trainingverifierssolvemath, hendrycks2021measuring} to interactive systems that generate structured feedback, Socratic questions, and adaptive multi-turn guidance~\cite{10.5555/3737916.3740637,10.1145/3627673.3679881}. Prior systems improve tutoring through Socratic dialogue generation~\cite{10.5555/3737916.3740637}, conversational mathematics and coding tutoring~\cite{10.1145/3627673.3679881, kargupta-etal-2024-instruct}, stepwise reasoning verification~\cite{daheim-etal-2024-stepwise}, fine-grained pedagogical intent modeling~\cite{petukhova-kochmar-2025-intent,pal-etal-2025-beyond}, strategy-aware steering~\cite{puech-etal-2025-towards}, and turn-level verifier-guided response selection~\cite{wang-etal-2025-training}. Related prompting work further shows that Socratic-style reasoning can improve LLM problem solving~\cite{pei-etal-2025-socratic}. Beyond generation quality, recent tutoring agents increasingly incorporate explicit student-state tracking, error diagnosis, and intent-conditioned response planning to make tutor behavior more interpretable and controllable~\cite{puech-etal-2025-towards,wang-etal-2025-training,petukhova-kochmar-2025-intent}. These developments have produced stronger pedagogical responses and more structured tutoring interactions across math, coding, and self-regulated learning settings~\cite{10.1145/3627673.3679881,wang-etal-2025-training,liu2026metaclassmetacognitivecoachinglearning, tiwari2026turning}.

\noindent\textbf{Teacher-student explanation and intervention efficiency.}
A parallel line of work studies what makes a teacher's communication useful to a student, a question long studied in intelligent tutoring through feedback, hinting, and human-tutoring effectiveness~\cite{vanlehn2011relative}. Early work distinguishes \emph{simulatability} (whether explanations help a student predict the teacher's answer) from direct teaching utility (where explanations help the student learn the task itself)~\cite{chandrasekaran-etal-2018-explanations, treviso-martins-2020-explanation, hase-bansal-2020-evaluating}. Recent work extends this question to LLMs: whether stronger models can teach weaker agents through explanations under communication or budget constraints~\cite{saha2023can}, whether human explanations can teach LLMs~\cite{lampinen-etal-2022-language}, and whether teacher knowledge can be distilled into student models via fine-tuning on teacher-generated rationales~\cite{magister-etal-2023-teaching, pmlr-v202-fu23d, wang2024tpd}. These approaches improve how a teacher \emph{communicates} when communication occurs, but assume that the teacher should communicate at every step. However, they generally do not model the cost of unnecessary communication, the trade-off between intervention frequency and student performance, or when the teacher should decide \emph{not} to intervene.

Taken together, prior work has made substantial progress in improving the quality, structure, and pedagogical intent of tutor responses. Most systems still treat tutoring as response generation triggered at every student turn. \textsc{DICE} instead makes intervention \emph{necessity} an explicit learned decision, separating \emph{whether} to intervene from \emph{what} action to take and \emph{how} to realize it.

\section{New Benchmark: \textsc{DICE-Bench}}
\label{sec:data}

Existing tutoring datasets~\cite{macina-etal-2023-mathdial,10.5555/3737916.3740637} treat tutoring as a response-generation task: each example pairs a tutoring context with a teacher utterance, and a model is trained to imitate that utterance. Two things are missing from this setup. First, problems appear in isolation, so there is no way to measure whether tutoring actually helps a student progress \emph{across} problems. Particularly, whether mastering an easy problem prepares the student to solve a harder one of the same skill, or whether the tutor's guidance generalizes beyond a single instance. Second, every annotated turn assumes the teacher should speak; the data contains no signal for when a tutor should simply acknowledge a correct answer and let the student continue. As a result, models trained on these datasets learn \emph{what} to say at every turn but never learn \emph{when} not to say anything. \textsc{DICE-Bench} addresses both gaps. It pairs each source problem with skill-preserving easier and harder variants, enabling \emph{session-level} evaluation of whether tutoring helps a student transfer from simpler to harder instances of the same concept. In addition, it labels a subset of tutoring states from the SocraTeach \cite{10.5555/3737916.3740637} with an explicit pedagogical action, giving models a direct learning signal for when to intervene.

We describe the source data (\S\ref{sec:source}), difficulty augmentation (\S\ref{sec:augmentation}), policy annotation (\S\ref{sec:annotation}), and benchmark summary (\S\ref{sec:dataset_stats}).

\subsection{Source Data}
\label{sec:source}

We build on SocraTeach~\cite{10.5555/3737916.3740637}, a synthetic tutoring corpus constructed from GSM8K~\cite{cobbe2021trainingverifierssolvemath} and MAWPS~\cite{koncel-kedziorski-etal-2016-mawps} math word problems. Each problem \(q\) in SocraTeach comes with a gold answer \(q^\ast\), a reasoning analysis, and a sequence of \(M\) subgoal steps \(\{g_1,\ldots,g_M\}\) that define the intended tutoring trajectory. Particularly, two properties make SocraTeach a useful source for intervention-calibrated tutoring research. First, the subgoal decompositions expose natural \emph{decision points} at which a tutor must decide how to respond. Second, its multi-dialogue structure, where each problem is paired with conversations across different simulated student profiles, provides diverse tutoring states for the same underlying problem. We therefore use SocraTeach as the source pool for constructing \textsc{DICE-Bench}.

\subsection{Skill-Preserving Difficulty Augmentation}
\label{sec:augmentation}

A central design goal of \textsc{DICE-Bench} is to enable the evaluation of tutoring policies across problems of varying difficulty that share the same underlying skill. For each source problem $q$, we generate two \emph{sibling variants}: an easier problem $q^{-}$ and a harder problem $q^{+}$, forming a \emph{problem triplet} $(q^{-}, q, q^{+})$. This triplet structure serves two purposes: it enables session-level evaluation where a tutor guides a student across incrementally harder problems under a fixed interaction budget, and it provides a controlled testbed for measuring whether intervention utility varies with problem difficulty.

\begin{figure}[htbp]
    \centering
\includegraphics[width=0.75\textwidth]{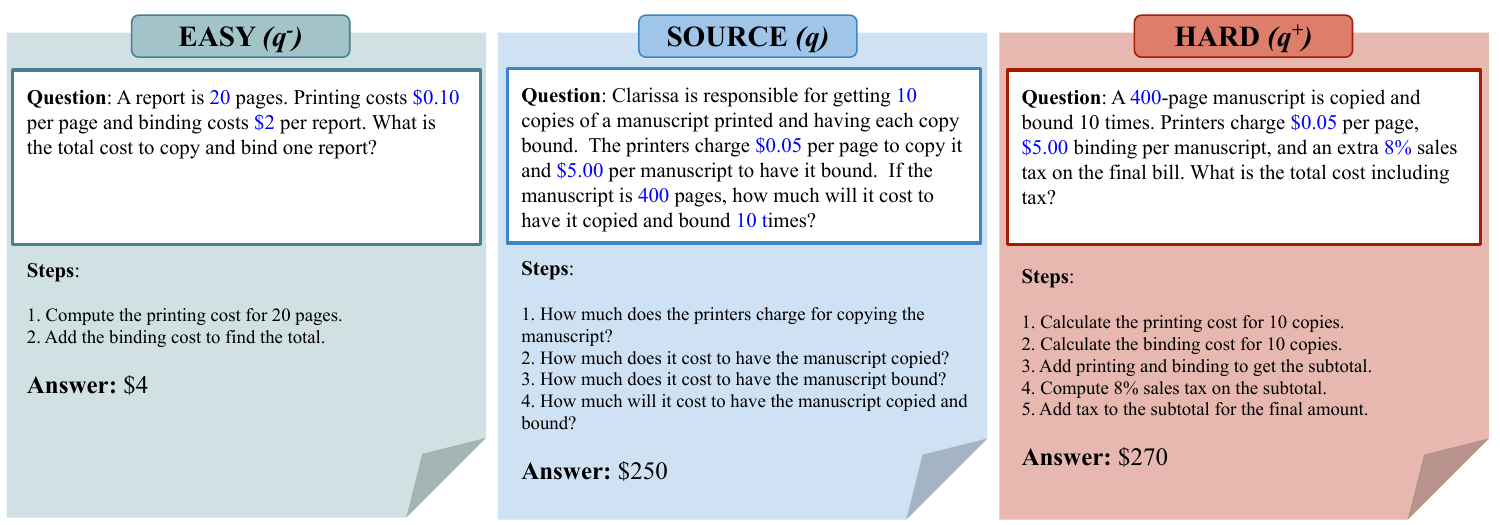}
   \caption{\footnotesize A qualitative example of \textsc{DICE-Bench} difficulty augmentation. Easy ($q^{-}$), Source ($q$), and Hard ($q^{+}$) variants preserve the same underlying skill while increasing reasoning and operational complexity.}    
    \label{fig:qualitative_difficulty_example}
\end{figure} 

\noindent\textbf{Principled difficulty variation.}
We ground difficulty variation in cognitive and computational accounts of problem complexity. Cognitive Load Theory (\abv{clt}) views difficulty as a combination of cognitive demands rather than a single scalar~\cite{sweller1994cognitive}; following complexity metrics for mathematical reasoning~\cite{fu2023complexitybased}, we vary problems along three axes: \textit{reasoning depth}, the number of sequential inference hops in the solution chain or \abv{dag}~\cite{wei2022chain}; \textit{operational complexity}, the arithmetic structure induced by adding, removing, or composing operation nodes~\cite{yin-neubig-2017-syntactic}; and \textit{linguistic load}, the parsing burden approximated by syntactic dependency depth~\cite{de2006generating}. Easy variants reduce one or more demands, while hard variants increase them, but all variants must preserve the same underlying mathematical skill.

\begin{wraptable}{r}{0.40\textwidth}
\vspace{-2.0em}
\centering
\scriptsize
\caption{\scriptsize Structural validation of generated difficulty variants in \textsc{DICE-Bench}. Means are computed over 10,273 matched triplets \((q^{-},q,q^{+})\); pairwise \(q^{-}\!\leq q\) and \(q\!\leq q^{+}\) differences are significant under paired Wilcoxon signed-rank tests.}
\label{tab:difficulty_validation}
\resizebox{0.39\textwidth}{!}{%
\begin{tabular}{lccccc}
\toprule
Metric & Easy & Source & Hard & E $\leq$ S & S $\leq$ H \\
\midrule
Steps & 2.23 & 3.09 & 3.63 & 95.9\% & 86.8\% \\
Operators & 1.96 & 3.27 & 4.49 & 88.9\% & 79.9\% \\
Equations & 2.33 & 2.93 & 4.80 & 87.5\% & 92.6\% \\
Numeric tokens & 7.00 & 10.71 & 14.19 & 94.3\% & 83.2\% \\
Question tokens & 29.38 & 47.18 & 47.76 & 97.4\% & 60.8\% \\
\bottomrule
\end{tabular}%
}
\vspace{-1.0em}
\end{wraptable}

\noindent\textbf{Generation procedure.}
For each $q$, we prompt a frontier language model to generate ($q^{-}$, $q^{+}$) that preserves the same skill while varying structural difficulty. Fig.~\ref{fig:qualitative_difficulty_example} shows a qualitative example of the triplet. As shown in Table~\ref{tab:difficulty_validation}, the generated variants exhibit the intended easy-source-hard structure across solution-level metrics, while question length shows weaker source-to-hard separation; details are provided in Appendix~\ref{app:dice_bench_aug}.

\subsection{Turn-Level Policy Annotation}
\label{sec:annotation}

\begin{wraptable}{l}{0.40\textwidth}
\vspace{-1.0em}
\centering
\scriptsize
\caption{\footnotesize Composition of \textsc{DICE-Bench}. Splits are by source problem ID; action distribution is reported over all variants.}
\label{tab:data_stats}
\resizebox{0.39\textwidth}{!}{%
\begin{tabular}{lrrr}
\toprule
& $\mathcal{D}_{\mathrm{src}}$ & $\mathcal{D}_{\mathrm{easy}}$ & $\mathcal{D}_{\mathrm{hard}}$ \\
\midrule
Provenance & dialogue & step & step \\
Decision points & 7{,}625 & 4{,}728 & 8{,}124 \\
\midrule
\multicolumn{4}{l}{\textit{Action distribution, all variants, $N=20{,}477$}} \\
\midrule
$a_0$ \;(\texttt{accept}) & \multicolumn{3}{c}{8{,}037 \;(39.2\%)} \\
$a_1$ \;(\texttt{arith.\ repair}) & \multicolumn{3}{c}{4{,}019 \;(19.6\%)} \\
$a_3$ \;(\texttt{handle\_q}) & \multicolumn{3}{c}{3{,}888 \;(19.0\%)} \\
$a_5$ \;(\texttt{refocus}) & \multicolumn{3}{c}{2{,}177 \;(10.6\%)} \\
$a_2$ \;(\texttt{concept.\ hint}) & \multicolumn{3}{c}{1{,}906 \;(9.3\%)} \\
$a_4$ \;(\texttt{justify}) & \multicolumn{3}{c}{450 \;(2.2\%)} \\
\bottomrule
\end{tabular}%
}
\vspace{-1.0em}
\end{wraptable}

\textsc{DICE-Bench} turns free-form tutoring data into explicit supervision for tutor action prediction, with dataset composition and action-label statistics summarized in Table~\ref{tab:data_stats}. For each source problem \(q\) and its skill-preserving variants \((q^{-},q^{+})\), we create decision states \(x\) and ask a frontier model to choose the appropriate tutor action \(a\in\mathcal{A}\), where \(\mathcal{A}=\{a_0,\ldots,a_5\}\). We obtain these states from two sources. First, for dialogue-derived states \(\mathcal{D}_{\mathrm{dial}}\), we take each turn in an existing SocraTeach conversation, provide the problem, dialogue history, current subgoal, tutor prompt, and student response, and annotate what the tutor should do next. Second, for step-variant states \(\mathcal{D}_{\mathrm{step}}\), where no original dialogue exists for \(q^{-}\) and \(q^{+}\), we build local subgoal-level states from the synthetic variant solution steps and annotate the next tutor action for each state. Each final instance is \((x_i,a_i,v_i,r_i)\), where \(v_i\in\{q^{-},q,q^{+}\}\) identifies the variant and \(r_i\in\{\mathcal{D}_{\mathrm{dial}},\mathcal{D}_{\mathrm{step}}\}\) identifies the annotation source. Additional annotation details are provided in Appendix~\ref{app:policy_annotation_details}.

\begin{wrapfigure}{r}{0.40\textwidth}
\vspace{-3.0em}
\centering
\includegraphics[width=0.35\textwidth]{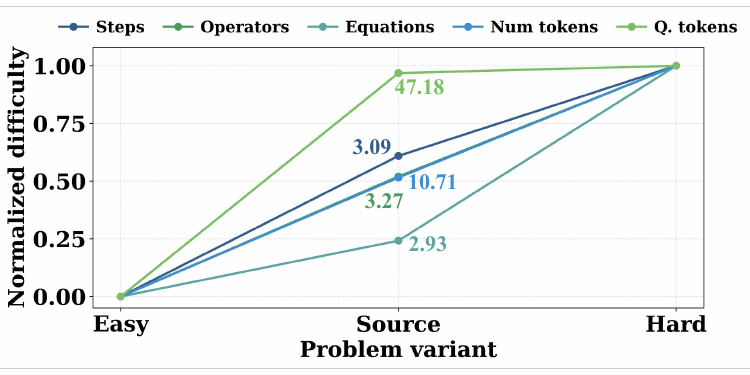}
\vspace{-0.8em}
\caption{\footnotesize Normalized difficulty trajectory across variants. Values are normalized per metric to highlight relative progression across the triplet.}
\label{fig:normalized_difficulty_trajectory}
\vspace{-1.0em}
\end{wrapfigure}

\subsection{\textsc{DICE-Bench} Summary}
\label{sec:dataset_stats}

Overall, \textsc{DICE-Bench} contains 10{,}273 matched problem triplets \((q^{-},q,q^{+})\) and 20{,}477 cleaned turn-level policy instances for action prediction. Tables~\ref{tab:difficulty_validation} and~\ref{tab:data_stats} summarize the benchmark composition, while Fig.~\ref{fig:normalized_difficulty_trajectory} visualizes the normalized difficulty trajectory across variants. This structure enables controlled evaluation of both session-level learning progress and turn-level intervention calibration.

\section{Method}
\label{sec:method}

\noindent\textbf{Preliminaries and Motivation.}
LLM tutors are typically trained as \emph{response generators}: given a tutoring context, produce the next teacher utterance. Systems such as SocraticLM show that LLMs can imitate multi-turn Socratic teaching, but this formulation assumes that a response is always needed. Effective tutoring, however, often begins with a prior decision: should the tutor step in, or should the student continue reasoning? This distinction is obscured by standard tutoring supervision, where teacher responses are often annotated or generated with access to the gold solution and intended reasoning path. Such \emph{privileged context} makes intervention appear naturally justified, since the annotator can immediately see whether the student is correct, partially correct, or drifting. At deployment time, however, the tutor must act from the observable dialogue state alone. Models trained to imitate answer-aware teacher responses may therefore inherit an intervention bias, producing guidance even when acknowledgment or non-intervention would be sufficient.

We recast tutoring as \emph{intervention-calibrated policy optimization}. Given a tutoring state \(x_t\), the tutor first selects a discrete pedagogical action \(a_t\in\mathcal{A}\), where \(a_0\) denotes non-intervention and \(\mathcal{A}^{+}=\mathcal{A}\setminus\{a_0\}\) denotes active interventions. Only when \(a_t\in\mathcal{A}^{+}\) does the system generate action-conditioned guidance. This separates three decisions conflated in standard response generation: (1) \emph{whether} to intervene, (2) \emph{what} pedagogical action to take, and (3) \emph{how} to realize that action in language. \textsc{DICE} formalizes tutoring as action selection (\S\ref{sec:formulation}), defines Intervention Value as counterfactual intervention utility (\S\ref{sec:iv}), audits supervision bias (\S\ref{sec:audit}), and optimizes an IV-calibrated policy for action-conditioned guidance generation (\S\ref{sec:iv_policy}). Architecture visualization is provided in Appendix~\ref{app:dicebench_architecture}.

\subsection{Tutoring as Action Selection}
\label{sec:formulation}

At turn \(t\), the tutor observes a state $x_t=(q,h_t,g_t,u_t,s_t),$
where \(q\) is the problem, \(h_t\) the dialogue history, \(g_t\) the current subgoal, \(u_t\) the tutor prompt, and \(s_t\) the student response. A tutoring policy $\pi_\theta:\mathcal{X}\rightarrow \Delta(\mathcal{A})$
selects a pedagogical action \(a_t\in\mathcal{A}\), with \(a_0\) denoting non-intervention and \(\mathcal{A}^{+}=\mathcal{A}\setminus\{a_0\}\) denoting active interventions. This makes the intervention decision explicit: the tutor must first decide whether to accept or take action before generating any teacher utterance.

\noindent\textbf{Supervised baseline.}
Given annotated data \(\mathcal{D}=\{(x_i,a_i)\}_{i=1}^{N}\), the standard supervised objective is
$
\mathcal{L}_{\textsc{sft}}
=
-\frac{1}{N}\sum_{i=1}^{N}\log\pi_\theta(a_i\mid x_i).
$
This treats every annotated action as equally trustworthy. However, when labels are produced with access to the gold solution, a pedagogically plausible intervention may still be unnecessary. This raises the central question: \emph{does each prescribed intervention improve student progress more than non-intervention?} We next define Intervention Value to answer this counterfactually.

\subsection{Intervention Value}
\label{sec:iv}

We introduce \emph{Intervention Value} (IV), a counterfactual measure of whether a tutor action improves student progress relative to non-intervention.

\begin{definition}[Intervention Value]\label{def:iv}
For a tutoring state $x$ and action $a\in\mathcal{A}$, the Intervention Value is
\begin{equation*}\label{eq:iv}
    \mathrm{IV}(x,a) \;=\; \mu(x,a) - \mu(x,a_0),
\end{equation*}
where $\mu(x,a) = \mathbb{E}\bigl[\,r\bigl(x,\,a,\,y\bigr)\bigr]$ is the expected student progress after the tutor takes action $a$, and $y\sim p_{\mathrm{student}}(\cdot\mid x, T(a,x))$ is the student's next response given the action-conditioned teacher utterance~$T(a,x)$.
\end{definition}

A positive IV indicates that intervention improves the student's trajectory; $\mathrm{IV}\approx 0$ indicates that the intervention provides no marginal benefit over letting the student continue; negative IV indicates that the intervention is actively worse than non-intervention. By construction, $\mathrm{IV}(x,a)\in[-1,1]$.

\noindent\textbf{Estimation via controlled rollouts.}
Since true student progress is not directly observable at each turn, we estimate \(\mu(x,a)\) through \emph{simulated rollouts}. For each state--action pair \((x,a)\), we (i)~render a controlled teacher utterance \(u_a=T(a,x)\) from an action-specific template, (ii)~sample \(K\) independent student continuations \(y^{(k)}\sim p_{\mathrm{student}}(\cdot\mid x,u_a)\) from a fixed student simulator, and (iii)~evaluate each continuation with a frontier judge model \(\mathcal{J}\). The judge assigns a progress label
\[
\ell^{(k)}\in\mathcal{L}_{\mathrm{prog}}
=
\{\ell_{\mathrm{help}},\ell_{\mathrm{part}},\ell_{\mathrm{none}}\},
\]
where \(\ell_{\mathrm{help}}\), \(\ell_{\mathrm{part}}\), and \(\ell_{\mathrm{none}}\) denote helpful progress, partial progress, and no progress, respectively.
which is mapped to a scalar progress score by
\begin{equation*}
\label{eq:progress_score}
r(\ell)=
\begin{cases}
1.0 & \text{if } \ell=\ell_{\mathrm{help}},\\
0.5 & \text{if } \ell=\ell_{\mathrm{part}},\\
0.0 & \text{if } \ell=\ell_{\mathrm{none}}.
\end{cases}
\end{equation*}
The estimated action quality and IV are then
\begin{equation*}\label{eq:iv_est}
    \hat{\mu}(x,a)
    =
    \frac{1}{K}\sum_{k=1}^{K} r\!\left(\ell^{(k)}\right),
    \qquad
    \widehat{\mathrm{IV}}(x,a)
    =
    \hat{\mu}(x,a)-\hat{\mu}(x,a_0).
\end{equation*}

Because both \(a\) and \(a_0\) are evaluated with the same simulator, judge, rollout budget, and rubric, \(\widehat{\mathrm{IV}}(x,a)\) asks not whether the judge is perfectly calibrated in absolute terms, but \textit{whether the intervention improves the student's next step more than letting the student continue}.

\subsection{Auditing Supervision with IV}
\label{sec:audit}

IV allows us to audit whether annotated tutor interventions provide marginal value over non-intervention. For each annotated intervention \((x_i,a_i)\) with \(a_i\in\mathcal{A}^{+}\), we assign
\begin{equation*}\label{eq:audit}
    Q_{\mathrm{IV}}(x_i,a_i)=
    \begin{cases}
        \textsc{helpful}  & \text{if } \widehat{\mathrm{IV}}(x_i,a_i) > \epsilon,\\
        \textsc{neutral}  & \text{if } |\widehat{\mathrm{IV}}(x_i,a_i)| \leq \epsilon,\\
        \textsc{harmful}  & \text{if } \widehat{\mathrm{IV}}(x_i,a_i) < -\epsilon,
    \end{cases}
\end{equation*}
where \(\epsilon\) defines a small margin around zero. We then aggregate \(Q_{\mathrm{IV}}\) by action type and annotation source. Action-level aggregation reveals which pedagogical moves are reliably useful versus over-prescribed, while source-level aggregation tests the information asymmetry hypothesis: interventions derived from answer-aware dialogue trajectories should have lower marginal value than step-level interventions constructed around local subgoals. We test this prediction empirically in \S\ref{sec:experiments}. Additional details on action-level and source-level IV aggregation are provided in Appendix~\ref{app:iv_audit_details}.

\subsection{IV-Calibrated Policy Optimization}
\label{sec:iv_policy}

Standard \abv{sft} treats every annotated action as equally reliable. We instead train the policy to trust intervention labels in proportion to their estimated marginal utility. For each training example \((x_i,a_i)\), we define
\begin{equation}\label{eq:weights}
    w_i =
    \begin{cases}
        1, & \text{if } a_i=a_0,\\[2pt]
        \max\!\bigl(\widehat{\mathrm{IV}}(x_i,a_i),0\bigr), & \text{if } a_i\in\mathcal{A}^{+}.
    \end{cases}
\end{equation}
Thus, non-intervention labels remain fully supervised, while active interventions are reinforced only when they improve progress over \(a_0\). Because IV estimates are obtained from simulated rollouts, we regularize the learned policy toward the original supervised policy \(\pi_{\textsc{sft}}\):
\begin{equation}\label{eq:dice}
    \mathcal{L}_{\textsc{dice}}
    =
    \underbrace{
        -\frac{1}{N}\sum_{i=1}^{N} w_i \log\pi_\theta(a_i\mid x_i)
    }_{\text{IV-weighted supervision}}
    \;+\;
    \beta\;\underbrace{
        \frac{1}{N}\sum_{i=1}^{N}
        D_{\mathrm{KL}}\!\bigl(\pi_\theta(\cdot\mid x_i)\;\big\|\;\pi_{\textsc{sft}}(\cdot\mid x_i)\bigr)
    }_{\text{proximity to supervised policy}}
\end{equation}
The first term corrects supervision bias by downweighting low-value interventions; the \abv{kl} term prevents over-correction from noisy IV estimates. Additional details provided in Appendix~\ref{app:iv_policy}.

\subsection{Cost-Sensitive Action-Conditioned Generation}
\label{sec:generation}

At inference time, \textsc{DICE} separates \emph{decision} from \emph{realization}: the policy first selects a pedagogical action, and the generator then realizes that action as a teacher utterance. Given state \(x_t\), we select
\begin{equation}
\label{eq:dice_cost}
    \hat{a}_t
    =
    \arg\max_{a\in\mathcal{A}}
    \left[
        \log \pi_\theta(a\mid x_t)
        -
        \lambda\,\mathbf{1}[a\neq a_0]
    \right],
\end{equation}
where \(a_0\) denotes non-intervention and \(\lambda\geq0\) is a validation-tuned intervention cost. Thus, active interventions must overcome a margin before the tutor interrupts the student. The selected action then conditions teacher generation. If \(\hat{a}_t=a_0\), the tutor emits a brief acknowledgment \(u_t=T_0(x_t)\) and advances. Otherwise,
$
    u_t \sim p_\phi(\cdot\mid x_t,\hat{a}_t),
    \;\hat{a}_t\in\mathcal{A}^{+}.
$
This factorization makes the intervention decision inspectable before free-form text is produced, allowing us to measure over-intervention, false acceptance, and action-selection errors directly at the policy level. We evaluate three inference rules: \textsc{DICE-MAP} selects the highest-probability action, \textsc{DICE-Cost} applies the intervention penalty in Eq.~\ref{eq:dice_cost}, and \textsc{DICE-Threshold} first thresholds the total probability of intervention before choosing an active action. Details are provided in Appendix~\ref{app:generation_details}; theoretical support for the cost-sensitive rule, IV-weighted KL objective, and provenance-based IV diagnostic is provided in Appendix~\ref{appendix:theory}.

\section{Experiments}
\label{sec:experiments}

We evaluate \textsc{DICE} on three settings: turn-level policy prediction, single-problem tutoring, and session-level tutoring over problem triplets \((q^{-},q,q^{+})\). For policy prediction, we compare pretrained, SFT, and IV-calibrated policies across Llama-3.1-8B-Instruct and Qwen-2.5-Instruct backbones. For tutoring, we first measure the student's unaided capability using student-only baselines: zero-shot, chain-of-thought (\abv{cot})~\cite{wei2022chain}, CoT with gold subgoal steps, and unassisted \(K\) attempts. We then compare against interactive tutoring baselines: Socratic prompting, MathDial fine-tuned tutoring~\cite{macina-etal-2023-mathdial}, SocraticLM~\cite{10.5555/3737916.3740637}, and StratL-inspired steering~\cite{puech-etal-2025-towards}. \textsc{DICE} is evaluated with three inference rules: \textsc{DICE-MAP} (Eq.~\ref{eq:dice_map}), \textsc{DICE-Cost} with validation-tuned \(\lambda\) (Eq.~\ref{eq:dice_cost}), and \textsc{DICE-Threshold} with validation-tuned \(\tau\) (Eq.~\ref{eq:dice_threshold}). Unless otherwise stated, all tutoring methods use Qwen-2.5-7B-Instruct as the teacher and Qwen-2.5-1.5B-Instruct as the student simulator. Full implementation details, decoding settings, and hyperparameters are provided in Appendix~\ref{app:experimental_details}.

\noindent\textbf{Evaluation.}
For policy prediction, we report action accuracy, action macro-F1, and intervention errors after collapsing actions into non-intervention \(a_0\) versus active intervention \(\mathcal{A}^{+}\). For tutoring, we report success rate, variant-level accuracy, average turns, over-intervention, false acceptance (\abv{facc}), and unsafe final-answer acceptance. Over-intervention means the tutor interrupts even when the student is already correct, while \abv{facc} means the tutor accepts an incorrect student response. Unsafe final-answer acceptance is the stricter case where the tutor accepts an incorrect response that looks like a final answer. Formal metric definitions are provided in Appendix~\ref{app:evaluation_metrics}.

\section{Results and Discussion}
\label{sec:results}

\noindent\textbf{Policy model comparison.}

\begin{wrapfigure}{l}{0.42\columnwidth}
\vspace{-0.8em}
\centering
\includegraphics[width=\linewidth]{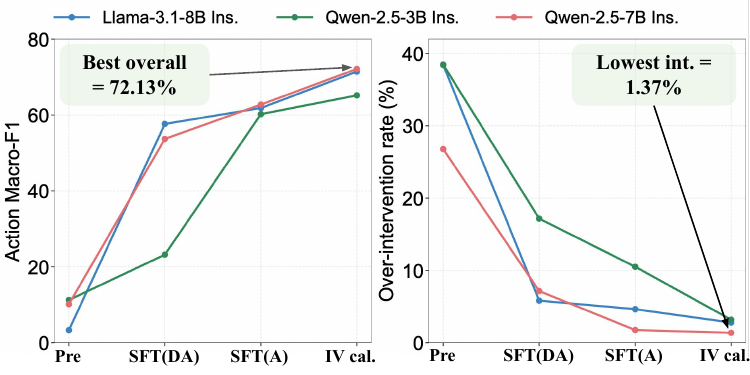}
\vspace{-0.8em}
\caption{\footnotesize Policy-model comparison. \textbf{Left: action macro-F1}. \textbf{Right: over-intervention rate}.}
\label{fig:policy_model_summary}
\vspace{-1.0em}
\end{wrapfigure}

We evaluate the policy model across two backbone families, Llama-3.1 and Qwen-2.5, and additionally include a smaller Qwen-2.5-3B model to test whether compact models can reliably predict pedagogical actions. As shown in Fig.~\ref{fig:policy_model_summary}, performance improves substantially from pretrained models to supervised policies, and improves further with IV-calibrated training. Among all variants, Qwen-2.5-7B-Instruct achieves the best overall balance of action prediction quality and intervention calibration, and is therefore used as the default policy backbone in the end-to-end experiments. Detailed per-model results are provided in Appendix~\ref{app:policy_results}.

\noindent\textbf{Understanding Student Capability}: Before evaluating tutors, we first measure the student model's unaided problem-solving ability on \(6k\) problems, with \(2k\) examples each from \(q^{-}\), \(q\), and \(q^{+}\), to ensure that tutoring gains are measured relative to a known student capability profile. We use Qwen-2.5-1.5B-Instruct as the fixed student simulator because it provides a balanced capability profile, being neither too weak nor already saturated; details are provided in Appendix~\ref{app:student_capability}.

\noindent\textbf{Single-problem tutoring.}
Table~\ref{tab:single_problem_results_compact} and Fig.~\ref{fig:single_problem_tutoring_summary} summarize the single-problem Socratic tutoring setting. Interactive tutoring substantially improves over unaided student baselines. However, existing tutors exhibit different failure modes: MathDial fine-tuning has the lowest \abv{facc} but over-intervenes heavily, while SocraticLM is turn-efficient but has a high \abv{facc}. \textsc{DICE} provides a stronger balance across metrics. \textsc{DICE-MAP} achieves the best overall accuracy (\(0.797\)), while \textsc{DICE-Threshold} achieves the best hard-variant accuracy and reduces over-intervention to near zero (\(0.002\)). These results suggest that explicit action selection improves tutoring effectiveness, and that calibrated intervention control is crucial for avoiding unnecessary teacher turns. We tune \textsc{DICE} inference hyperparameters on a 150-problem validation sweep; details are provided in Appendix~\ref{app:dice_hparam_sweep}.

\begin{figure}[htbp]
\centering
\begin{minipage}{0.60\columnwidth}
    \centering
    \scriptsize
    \captionof{table}{\footnotesize Single-problem results. \(q^{-}\), \(q\), and \(q^{+}\) denote easy, source, and hard variants. Best values are bolded and second-best values are underlined.}
    \label{tab:single_problem_results_compact}
    \renewcommand{\arraystretch}{1.03}
    \setlength{\tabcolsep}{1.15pt}
    \begin{tabular}{lcccccccc}
    \toprule
    \rowcolor{gray!25}
    \textbf{Method} &
    \textbf{\(K\)} &
    \textbf{Acc.} \(\uparrow\) &
    \textbf{\(q^{-}\)} \(\uparrow\) &
    \textbf{\(q\)} \(\uparrow\) &
    \textbf{\(q^{+}\)} \(\uparrow\) &
    \textbf{Turn} \(\downarrow\) &
    \textbf{\abv{oint}} \(\downarrow\) &
    \textbf{\abv{facc}} \(\downarrow\) \\
    \midrule

    \multicolumn{9}{l}{\cellcolor{gray!12}\textbf{Student-only Q\&A}} \\
    Zero-shot & 1 & 0.231 & 0.436 & 0.184 & 0.072 & -- & -- & -- \\
    CoT & 1 & 0.240 & 0.436 & 0.196 & 0.088 & -- & -- & -- \\
    CoT+Steps & 1 & 0.256 & 0.488 & 0.164 & 0.116 & -- & -- & -- \\
    \(K\)-Att. & 8 & 0.274 & 0.504 & 0.224 & 0.096 & 1.402 & -- & -- \\

    \midrule
    \multicolumn{9}{l}{\cellcolor{gray!12}\textbf{Socratic tutoring}} \\
    StratL & 8 & 0.703 & 0.804 & 0.784 & 0.512 & 1.559 & 0.381 & 0.394 \\
    Ft. MD & 8 & 0.550 & 0.668 & 0.616 & 0.368 & 2.653 & 0.673 & \textbf{0.133} \\
    SP & 8 & 0.762 & 0.920 & 0.804 & 0.564 & 2.875 & 0.486 & 0.391 \\
    SLM & 8 & 0.700 & 0.816 & 0.760 & 0.524 & \textbf{1.165} & 0.190 & 0.546 \\

    \rowcolor{oursblue}
    \textsc{DICE-MAP} & 8 & \textbf{0.797} & \underline{0.916} & \textbf{0.862} & \underline{0.599} & 1.845 & 0.020 & 0.369 \\
    \rowcolor{oursblue}
    \textsc{DICE-COST(\(\lambda=0.25\))} & 8 & 0.773 & 0.912 & \underline{0.852} & 0.539 & 1.773 & \underline{0.008} & 0.386 \\
    \rowcolor{oursblue}
    \textsc{DICE-THR.(\(\tau=0.8\))} & 8 & \underline{0.787} & \textbf{0.923} & 0.812 & \textbf{0.614} & \underline{1.624} & \textbf{0.002} & \underline{0.335} \\

    \bottomrule
    \end{tabular}%
\end{minipage}
\hfill
\begin{minipage}{0.38\columnwidth}
    \centering
    \includegraphics[width=\linewidth]{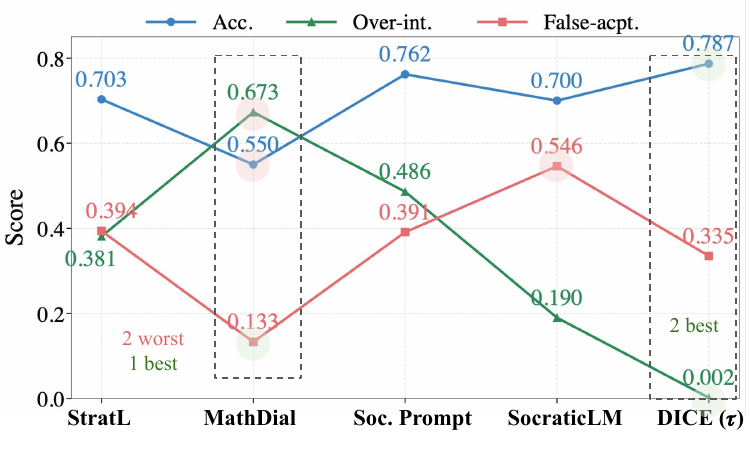}
    \vspace{-0.8em}
    \caption{\footnotesize Single-problem tutoring trade-off. \textsc{DICE-Threshold} maintains high accuracy while sharply reducing over-intervention.}
    \label{fig:single_problem_tutoring_summary}
\end{minipage}

\vspace{-0.8em}
\end{figure}

\noindent\textbf{Session-level tutoring.}
Table~\ref{tab:session_level_results} evaluates the central setting of \textsc{DICE}: guiding the same student through a problem triplet \((q^{-}, q, q^{+})\) under a shared \(20\)-turn budget. This setting is more demanding than single-problem tutoring because the tutor must not only solve an isolated problem, but also allocate interventions efficiently across increasing difficulty. Existing Socratic tutoring baselines struggle in this setting. MathDial fine-tuned tutor is conservative and achieves low unsafe final-answer acceptance (\(0.008\)), but solves all three variants in only \(0.080\) of sessions and over-intervenes in \(0.904\) of sessions. Socratic prompting and StratL improve success, but still use most of the turn budget and exhibit high over-intervention. SocraticLM is stronger than these baselines in turn efficiency, but it has the highest unsafe final-answer acceptance (\(0.400\)).

\textsc{DICE} provides a substantially better session-level trade-off. \textsc{DICE-Cost} achieves the best All-3 success rate (\(0.564\)), solves the most variants per session (\(2.324\) out of \(3\)), and uses the fewest turns on average (\(11.940\)). \textsc{DICE-Threshold} achieves the best easy-variant success (\(0.948\)) and the lowest over-intervention rate (\(0.024\)). All three \textsc{DICE} variants reduce unsafe final-answer acceptance to \(0\), showing that the acceptance verifier and final-answer guard prevent the most harmful accept decisions. Overall, the session results support the main hypothesis: decoupling response generation from intervention decisions enables tutors to preserve or improve student success while using fewer turns and avoiding unnecessary intervention. Detailed discussion in Appendix \ref{app:verifier_guard_discussion}.

\begin{wraptable}{r}{0.40\textwidth}
\vspace{-0.8em}
\centering
\tiny
\caption{\footnotesize Session-level ablation of \textsc{DICE}. The full system uses cost-sensitive action selection, an acceptance verifier, and a final-answer guard.}
\label{tab:session_ablation}
\renewcommand{\arraystretch}{1.08}
\setlength{\tabcolsep}{2.2pt}
\begin{tabular}{lccccc}
\toprule
\rowcolor{gray!25}
\textbf{Method} &
\textbf{All3} \(\uparrow\) &
\textbf{Vars.} \(\uparrow\) &
\textbf{Turn} \(\downarrow\) &
\textbf{OInt} \(\downarrow\) &
\textbf{UAcc} \(\downarrow\) \\
\midrule
\rowcolor{oursblue}
Full
& \textbf{0.564} & \textbf{2.324} & 11.940 & 0.028 & \textbf{0.000} \\
w/o Cost
& 0.532 & 2.284 & 13.048 & 0.036 & \textbf{0.000} \\
w/o Guard
& 0.552 & 2.300 & 11.720 & 0.026 & 0.032 \\
w/o Verifier
& 0.520 & 2.240 & \textbf{10.850} & \textbf{0.020} & 0.291 \\
w/o All
& 0.216 & 1.496 & 18.048 & 0.880 & 0.076 \\
\bottomrule
\end{tabular}
\vspace{-1.0em}
\end{wraptable}

\begin{figure}[htbp]
    \centering
    \includegraphics[width=0.98\textwidth]{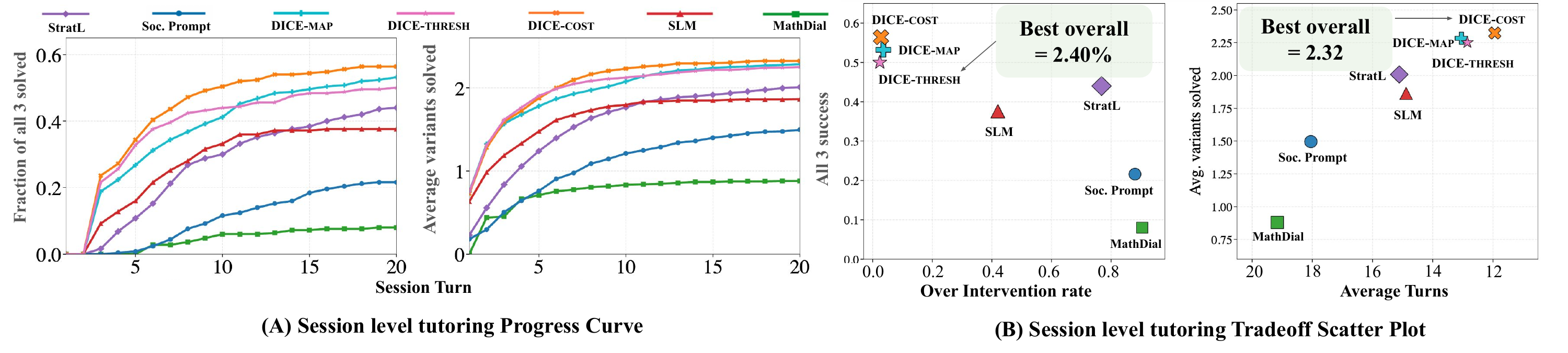}
    \caption{\footnotesize Session-level tutoring dynamics on \textsc{DICE-Bench}. Left: Progress curves showing the fraction of sessions where all three variants are solved and the average number of variants solved as the shared turn budget is used. Right: Trade-off plots comparing all-3 success against over-intervention and average variants solved against average turns. \textsc{DICE-Cost} achieves the strongest session-level trade-off, solving more variants in fewer turns while keeping over-intervention near zero.}
\label{fig:session_level_dynamics}
\end{figure}

Fig.~\ref{fig:session_level_dynamics} shows that \textsc{DICE} uses the shared turn budget more effectively than response-only tutors. \textsc{DICE-Cost} achieves the best overall session trade-off, solving the most variants in fewer turns while keeping over-intervention near zero.

\noindent\textbf{Ablation.}
Table~\ref{tab:session_ablation} ablates each \textsc{DICE} component independently. The learned action policy is the dominant contributor: removing it collapses session success from \(0.564\) to \(0.216\) and increases over-intervention from \(0.028\) to \(0.880\). The acceptance verifier and final-answer guard provide complementary safety: removing the verifier increases unsafe acceptance from \(0.000\) to \(0.291\). Cost-sensitive selection improves turn efficiency, reducing average turns from \(13.048\) to \(11.940\) while preserving safety.


\begin{table*}[t]
\centering
\footnotesize
\caption{\footnotesize \textbf{Session-level tutoring results on \textsc{DICE-Bench}}. Each session contains the triplet \((q^{-},q,q^{+})\) under a shared 20-turn budget. All-3 reports full triplet success, Vars. is the average number of variants solved, and UnsafeAcc. is unsafe final-answer acceptance. Higher is better for success metrics and Vars.; lower is better for Turns, OverInt., and UnsafeAcc. Best values are bolded and second-best values are underlined.}
\label{tab:session_level_results}
\renewcommand{\arraystretch}{1.12}
\setlength{\tabcolsep}{3.8pt}
\begin{tabular}{llcccccccc}
\toprule
\rowcolor{gray!25}
\textbf{Mode} &
\textbf{Method} &
\textbf{All-3} \(\uparrow\) &
\textbf{\(q^{-}\)} \(\uparrow\) &
\textbf{\(q\)} \(\uparrow\) &
\textbf{\(q^{+}\)} \(\uparrow\) &
\textbf{Vars.} \(\uparrow\) &
\textbf{Turns} \(\downarrow\) &
\textbf{OverInt.} \(\downarrow\) &
\textbf{UnsafeAcc.} \(\downarrow\) \\
\midrule

\multirow{4}{*}{\cellcolor{white}\textbf{Soc. Tut.}}
& StratL
& 0.440 & 0.868 & 0.700 & 0.440 & 2.008 & 15.112 & 0.768 & 0.272 \\
& Ft. MathDial
& 0.080 & 0.516 & 0.284 & 0.080 & 0.880 & 19.164 & 0.904 & \underline{0.008} \\
& Soc. Prompting
& 0.216 & 0.772 & 0.508 & 0.216 & 1.496 & 18.048 & 0.880 & 0.076 \\
& SocraticLM
& 0.376 & 0.844 & 0.644 & 0.376 & 1.864 & 14.884 & 0.420 & 0.400 \\

\midrule

\rowcolor{oursblue}
\textbf{\textsc{DICE}}
& MAP
& \underline{0.532} & \underline{0.928} & \underline{0.824} & \underline{0.532} & 2.284 & 13.048 & 0.036 & \textbf{0.000} \\
\rowcolor{oursblue}
& Cost \((\lambda=0.25)\)
& \textbf{0.564} & \underline{0.928} & \textbf{0.832} & \textbf{0.564} & \textbf{2.324} & \textbf{11.940} & \underline{0.028} & \textbf{0.000} \\
\rowcolor{oursblue}
& Thres. \((\tau=0.8)\)
& 0.500 & \textbf{0.948} & 0.804 & 0.500 & \underline{2.252} & \underline{12.852} & \textbf{0.024} & \textbf{0.000} \\

\bottomrule
\end{tabular}%
\end{table*}

\section{Conclusion}
\label{sec:conclusion}

We introduced \textsc{DICE}, a decision-aware framework for LLM tutoring that separates the decision of whether to intervene from the generation of the tutor response. Our results show that fluent Socratic response generation alone is insufficient: existing tutors can improve student accuracy, but often do so with excessive intervention or unsafe acceptance. By training an explicit action policy with Intervention Value and calibrating inference through cost and threshold-based decision rules, \textsc{DICE} improves session-level tutoring while sharply reducing over-intervention and unsafe final-answer acceptance. These findings suggest that scalable LLM tutoring requires not only better pedagogical language generation, but also calibrated decisions about when guidance is actually needed.

\newpage
\bibliographystyle{plain}
\bibliography{ref}

\clearpage
\appendix
\footnotesize

\startcontents[appendix]

{
\hypersetup{
    pdfborder={0 0 0}
}

\section*{Contents of the Appendix}
\printcontents[appendix]{}{1}{}
}

\clearpage
\normalsize

\section{\textsc{DICE-Bench} and \textsc{DICE} Architecture}
\label{app:dicebench_architecture}

\begin{figure}[htbp]
    \centering
    \includegraphics[width=0.98\textwidth]{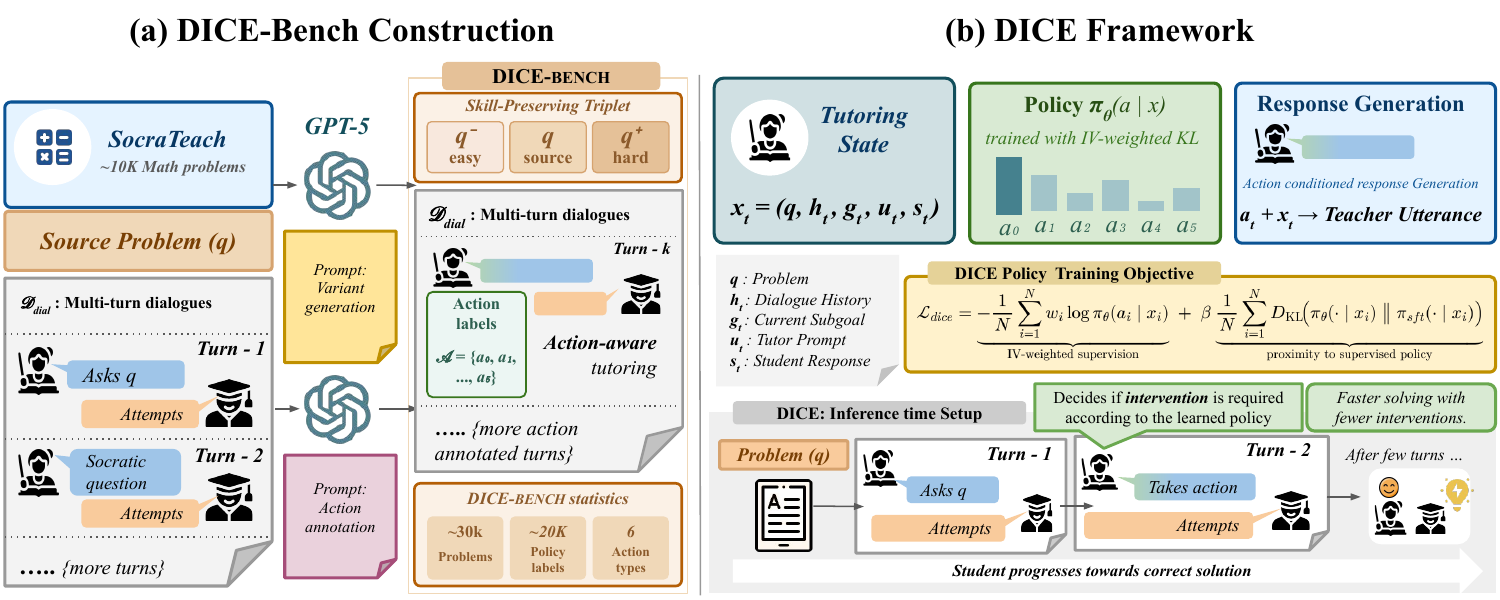}
    \caption{\footnotesize Overview of \textsc{DICE-Bench} construction and the \textsc{DICE} tutoring framework. Panel (a) shows how SocraTeach math problems and multi-turn Socratic dialogues are transformed into skill-preserving problem triplets \((q^{-}, q, q^{+})\) and turn-level action labels. Panel (b) shows the \textsc{DICE} pipeline: a tutoring state is mapped to an explicit pedagogical action policy trained with IV-weighted KL regularization, and the selected action conditions response generation during tutoring.}
    \label{fig:dice_overview}
\end{figure}

\section{DICE-Bench Details}
\label{app:dice_bench_aug}

\subsection{Generation Procedure and Structural Validation}
\label{app:generation_validation}

For each source problem, we prompt a frontier language model with the question, gold analysis, answer, and subgoal decomposition to generate a shared skill descriptor, an easy variant, and a hard variant. The easy variant reduces structural difficulty through fewer steps, simpler quantities, or more direct phrasing, while the hard variant adds a reasoning dependency or more complex operation without changing the underlying skill. Each variant includes its own question, analysis, answer, and subgoal sequence.

\begin{figure}[htbpt]
    \centering
    \begin{subfigure}[t]{0.32\textwidth}
        \centering
        \includegraphics[width=\linewidth]{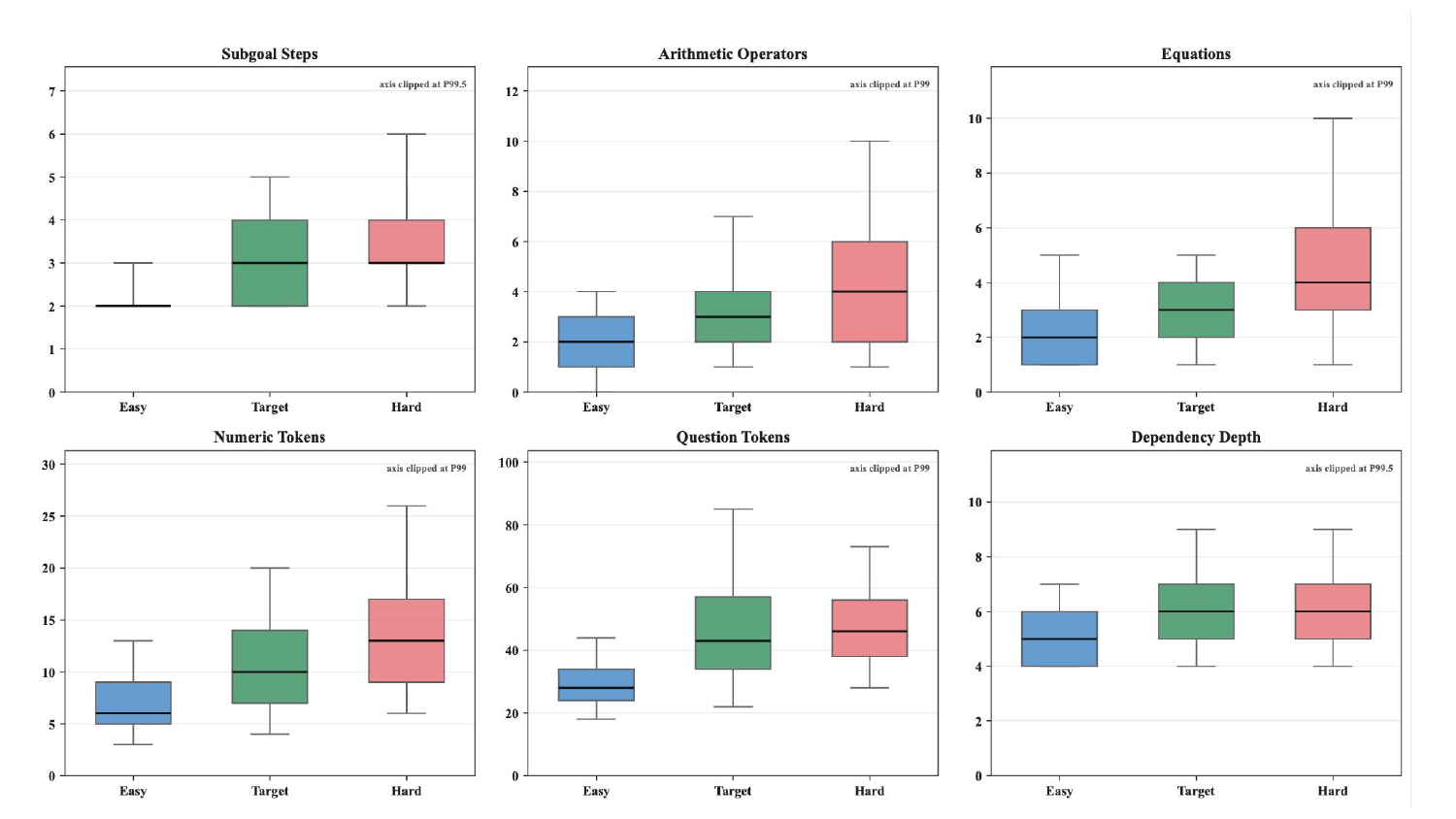}
        \caption{\scriptsize Distribution of structural metrics.}
        \label{fig:difficulty_box_app}
    \end{subfigure}
    \hfill
    \begin{subfigure}[t]{0.32\textwidth}
        \centering
        \includegraphics[width=\linewidth]{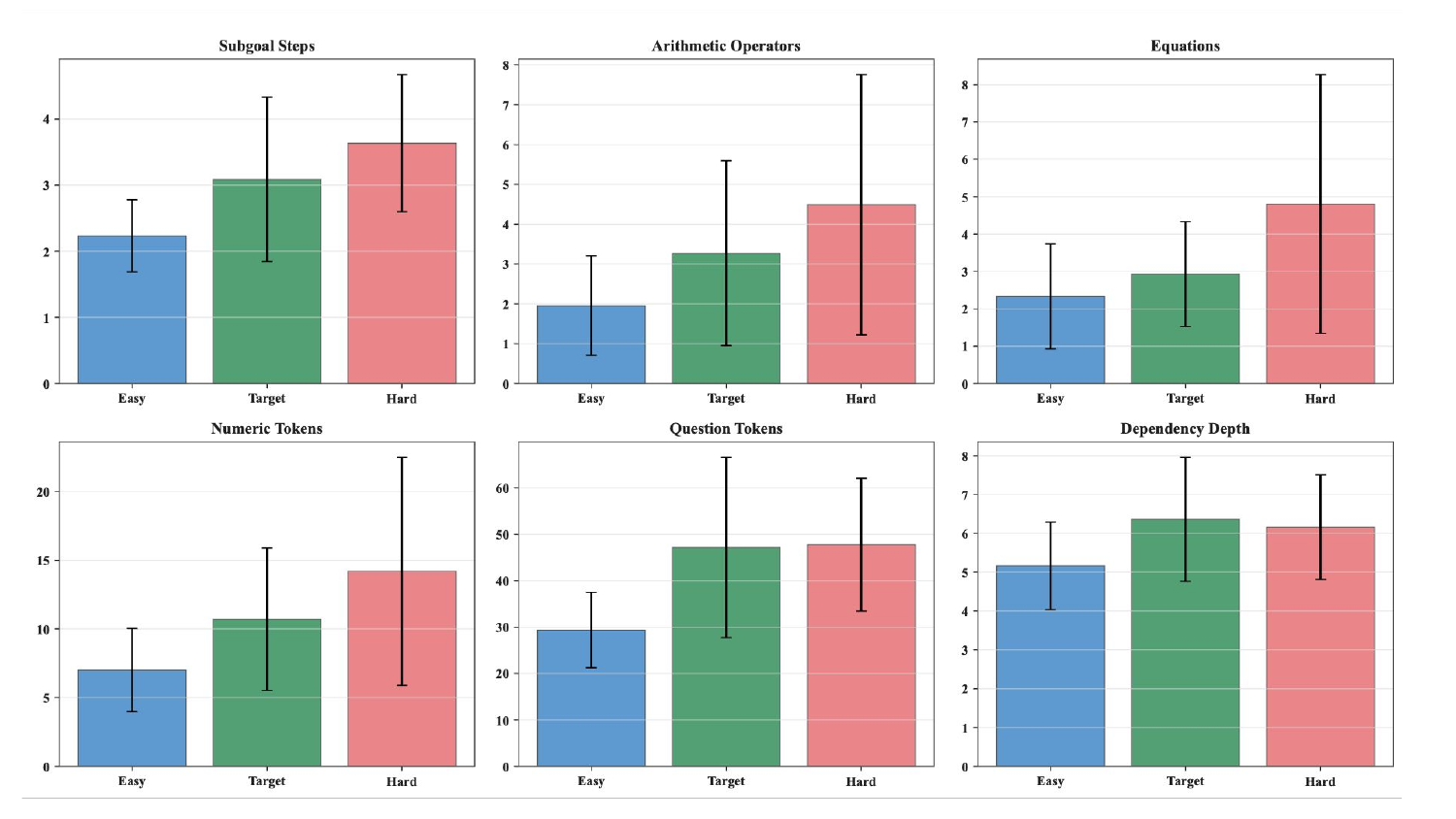}
        \caption{\scriptsize Mean structural metrics.}
        \label{fig:difficulty_bar_app}
    \end{subfigure}
    \hfill
    \begin{subfigure}[t]{0.32\textwidth}
        \centering
        \includegraphics[width=\linewidth]{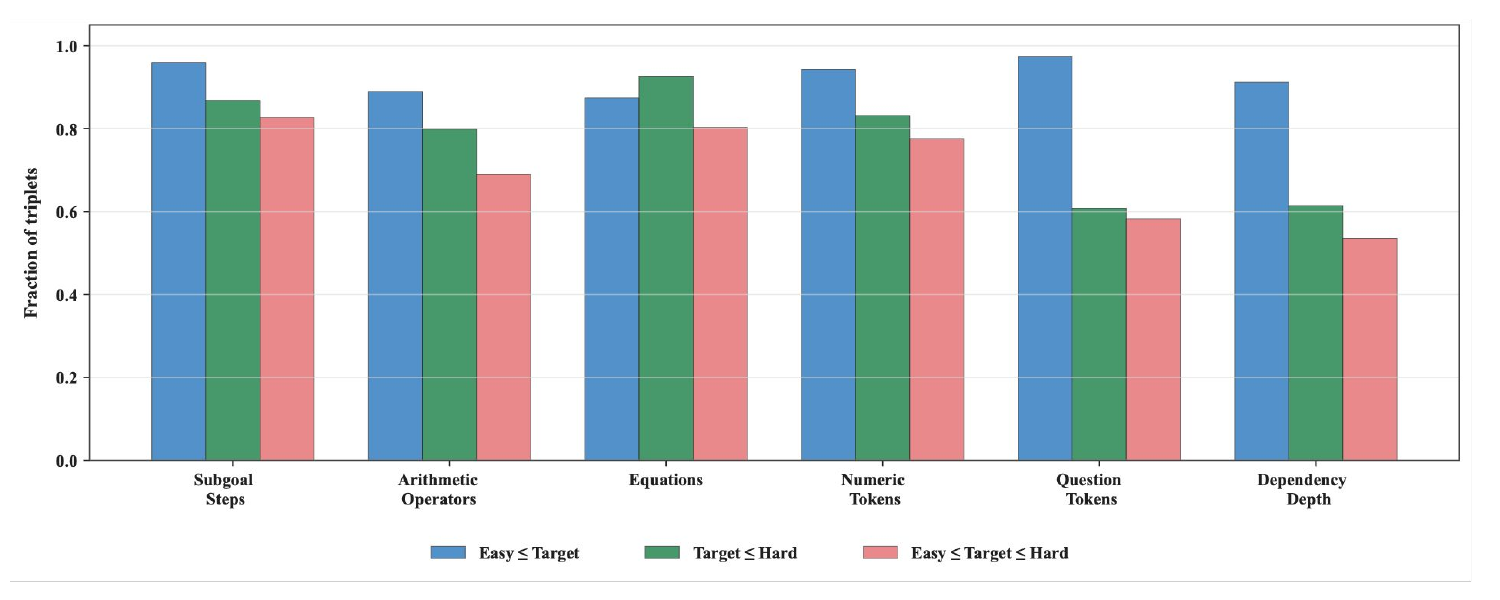}
        \caption{\scriptsize Ordering compliance.}
        \label{fig:difficulty_ordering_app}
    \end{subfigure}
    \caption{\small Structural validation of skill-preserving difficulty variants in \textsc{DICE-Bench}. Easy, source (target), and hard variants are compared across solution-structure and surface-form metrics.}
    \label{fig:difficulty_validation_app}
\end{figure}

The variants are generated independently of tutoring dialogues, so difficulty variation is not tied to teacher behavior or dialogue style. We validate the resulting 10,273 matched triplets using observable structural proxies. As shown in Fig.~\ref{fig:difficulty_box_app}, the triplets form a consistent difficulty ladder across solution-structure metrics. Fig.~\ref{fig:difficulty_bar_app} shows that easy variants reduce subgoal count, arithmetic operators, equations, and numeric quantities, while hard variants increase these metrics. Fig.~\ref{fig:difficulty_ordering_app} further reports the fraction of triplets satisfying the expected easy--source--hard ordering. Dependency depth and question length show weaker source-to-hard separation, suggesting that harder variants primarily increase mathematical structure rather than surface linguistic complexity.

\noindent\textbf{Quality control.}
After generating \(q^{-}\) and \(q^{+}\), we apply two checks for correctness and difficulty validity. First, we use an LLM council of three frontier judges, Claude Sonnet 4.6, Kimi K2.5 Thinking, and Qwen 3.5-122B, to verify arithmetic correctness, answer consistency, and whether the generated analysis supports the stated answer. A variant is retained only when the council finds the problem self-contained, solvable, and arithmetically consistent. Second, we manually audit a random subset of 50 triplets with two independent judges. Each judge rates \emph{solution accuracy} and \emph{difficulty alignment} on a 5-point Likert scale, where solution accuracy measures whether the question, reasoning, and answer are mutually consistent, and difficulty alignment measures whether \(q^{-}\) is easier than \(q\) and \(q^{+}\) is harder while preserving the same skill. The two judges show strong agreement, with an intraclass correlation coefficient of \(0.85\), supporting the reliability of the manual validation.

\subsection{System Prompt for Skill-Preserving Augmentation}
\label{app:augmentation_prompt}

We use a structured prompt to generate easy and hard sibling variants for each source problem. The prompt enforces skill preservation, mathematical correctness, and structured outputs for downstream validation. Fig.~\ref{fig:dicebench_augmentation_prompt} shows the prompt template.

\begin{figure}[htbp]
\centering
\begin{minipage}[t]{0.95\columnwidth}
\begin{tcolorbox}[
  promptbox,
  title={Skill-Preserving Difficulty Augmentation Prompt}
]
\scriptsize

{\bf System Prompt:} {\it
You are an expert math curriculum designer. Given a source math word problem, generate two sibling variants that preserve the same underlying mathematical skill: one easier variant and one harder variant. The easy variant should reduce structural difficulty through fewer reasoning steps, simpler quantities, or more direct wording. The hard variant should add a meaningful reasoning dependency, composed operation, or less direct formulation. All generated problems must be mathematically correct, self-contained, and solvable from the information provided.
}\vspace{1mm}

{\bf User Prompt:} Given the source problem below:\vspace{0.5mm}

\begin{center}
{\tt <SOURCE\_QUESTION>}
\end{center}

with reference analysis, answer, and subgoal decomposition:\vspace{0.5mm}

\begin{center}
{\tt <SOURCE\_ANALYSIS>, <SOURCE\_ANSWER>, <SOURCE\_STEPS>}
\end{center}

Generate a valid JSON object with the following fields:\vspace{0.5mm}
\begin{enumerate}
    \item {\tt target\_skill}: a concise description of the shared mathematical skill.
    \item {\tt easy\_question}: an easier problem testing the same skill.
    \item {\tt easy\_analysis}, {\tt easy\_answer}, {\tt easy\_steps}: a correct solution and subgoal sequence for the easy variant.
    \item {\tt hard\_question}: a harder problem testing the same skill.
    \item {\tt hard\_analysis}, {\tt hard\_answer}, {\tt hard\_steps}: a correct solution and subgoal sequence for the hard variant.
    \item Do not introduce a new mathematical topic, unrelated story context, or ambiguous wording.
    \item Return only valid JSON with no explanations.
\end{enumerate}

\end{tcolorbox}
\end{minipage}

\vspace{-1mm}
\caption{\small Prompt template used for skill-preserving easy/hard problem augmentation in \textsc{DICE-Bench}.}
\label{fig:dicebench_augmentation_prompt}
\end{figure}

\subsection{Turn-Level Policy Annotation Details}
\label{app:policy_annotation_details}

For dialogue-derived annotations \(\mathcal{D}_{\mathrm{dial}}\), we extract decision points from existing SocraTeach conversations. A teacher--student exchange defines \(x_t=(q,h_t,g_t,u_t,s_t)\), where \(h_t\) is the prior dialogue, \(g_t\) the current subgoal, \(u_t\) the tutor prompt, and \(s_t\) the student response. A frontier (GPT-5) annotator assigns \(a_t\in\mathcal{A}\), along with auxiliary diagnostics \(z_t\) for local correctness, global trajectory, and reasoning issue. We use \(z_t\) for inspection and quality control, and \(a_t\) as the policy target. For step-variant annotations \(\mathcal{D}_{\mathrm{step}}\), no dialogue is available for \(q^{-}\) and \(q^{+}\); we therefore construct local states from each variant's subgoal sequence by generating student responses that cover common tutoring conditions, including correct progress, arithmetic error, conceptual confusion, skipped subgoal, weak justification, and student question. Unlike \(\mathcal{D}_{\mathrm{dial}}\), these labels are not conditioned on an existing teacher-led trajectory; they reflect isolated subgoal-level decisions. This distinction supports our later Intervention Value analysis, which asks whether labels derived from different annotation sources differ in their marginal utility over non-intervention.

\noindent\textbf{Quality control.}
Because policy learning is sensitive to noisy action labels, we filter annotations for semantic consistency before training. We standardize label formats, remove duplicate student states, and discard cases where the diagnosis and action disagree, such as accepting an incorrect response or prescribing arithmetic repair for a locally correct step. This yields 20{,}477 clean decision points from 25{,}121 raw annotations. We also merge or remove rare unsupported actions, resulting in the six-action space.

\subsection{System Prompt for Turn-Level Policy Annotation}
\label{app:annotation_prompt}

We use a structured prompt to convert tutoring states into action-labeled policy instances. For dialogue-derived source states, the input includes the dialogue history and current teacher--student exchange. For step-variant states, the same schema is used, but the dialogue history is omitted and the state is constructed directly from the current subgoal. Fig.~\ref{fig:dicebench_annotation_prompt} shows the prompt template.

\begin{figure}[htbp]
\centering
\begin{minipage}[t]{0.95\columnwidth}
\begin{tcolorbox}[
  promptbox,
  title={Turn-Level Tutor-Policy Annotation Prompt}
]
\scriptsize

{\bf System Prompt:} {\it
You are an expert annotator for trajectory-aware math tutoring. Given a tutoring state, diagnose the student's response relative to the current subgoal and select the next tutor action. Distinguish local correctness from global solution trajectory: a response may be locally incomplete even if it moves toward the final answer, or globally ahead while failing to answer the current subgoal. First diagnose the student state, then choose the smallest appropriate pedagogical action. Return only valid JSON.
}\vspace{1mm}

{\bf User Prompt:} Given the tutoring state below:\vspace{0.5mm}

\begin{center}
{\tt <QUESTION>, <GOLD\_ANALYSIS>, <GOLD\_ANSWER>, <SUBGOAL\_STEPS>}
\end{center}

with the current tutoring context:\vspace{0.5mm}

\begin{center}
{\tt <DIALOGUE\_HISTORY>, <CURRENT\_SUBGOAL>, <CURRENT\_TEACHER\_QUESTION>, <CURRENT\_STUDENT\_RESPONSE>}
\end{center}

Return a valid JSON object with the following fields:\vspace{0.5mm}
\begin{enumerate}
    \item {\tt current\_subgoal}: the subgoal most relevant to the current state.
    \item {\tt local\_correctness}: one of {\tt correct}, {\tt partial}, or {\tt incorrect}.
    \item {\tt global\_trajectory}: one of {\tt on\_track}, {\tt ahead}, {\tt off\_track}, or {\tt stalled}.
    \item {\tt reasoning\_issue}: one of {\tt none}, {\tt arithmetic\_error}, {\tt subgoal\_skip}, {\tt conceptual\_gap}, {\tt weak\_justification}, or {\tt student\_question}.
    \item {\tt intervention\_needed}: one of {\tt none}, {\tt light}, or {\tt strong}.
    \item {\tt best\_action}: one of {\tt accept\_and\_advance}, {\tt arithmetic\_repair}, {\tt conceptual\_hint}, {\tt handle\_question}, {\tt ask\_for\_justification}, or {\tt refocus\_subgoal}.
    \item {\tt rationale}: a concise justification for the selected action.
    \item Return only valid JSON with no explanations.
\end{enumerate}

\end{tcolorbox}
\end{minipage}

\vspace{-1mm}
\caption{\small Prompt template used for turn-level tutor-policy annotation in \textsc{DICE-Bench}. Dialogue-derived source states include dialogue history and the observed teacher--student exchange; step-variant states use the same schema but omit dialogue context and are constructed directly from the current subgoal.}
\label{fig:dicebench_annotation_prompt}
\end{figure}

\subsection{Qualitative Examples of Tutor-Policy Decision Points}
\label{app:qual_policy_examples}

Table~\ref{tab:qual_policy_examples} shows representative \textsc{DICE-Bench} decision points. Each example includes the problem context, dialogue state available to the annotator, the current student response, and the resulting tutor-policy label. These cases illustrate why action prediction requires modeling the student's state relative to the current subgoal, not merely checking whether a number appears correct.

\begin{table}[htbp]
\centering
\scriptsize
\caption{\footnotesize Qualitative examples of turn-level tutor-policy annotations for one source problem in \textsc{DICE-Bench}. We show three different dialogue rollouts for the same problem and the unique corresponding tutor actions annotated by a frontier model (GPT-5).}
\label{tab:qual_policy_examples}
\renewcommand{\arraystretch}{1.18}
\setlength{\tabcolsep}{4pt}
\begin{tabularx}{\textwidth}{p{0.16\textwidth}Y p{0.18\textwidth}Y}
\toprule
\rowcolor{gray!25}
\textbf{Action} & \textbf{Problem and dialogue context} & \textbf{Diagnosis} & \textbf{Rationale} \\
\midrule

\actionbadge{acceptbg}{acceptfg}{Accept and advance} &
\textbf{Problem:} John plants a plot of \(3\) trees by \(4\) trees. Each tree gives \(5\) apples, and each apple sells for \(\$0.5\). How much money does John make?

\textbf{T:} How many trees does John plant?

\textbf{S:} John plants \(3\times4=12\) trees. &
-- Correct

-- On track

-- No issue &
The student answers the current subgoal correctly and shows the needed reasoning, so no intervention is necessary. \\

\midrule

\actionbadge{justifybg}{justifyfg}{Ask for justification} &
\textbf{Problem:} John plants a plot of \(3\) trees by \(4\) trees. Each tree gives \(5\) apples, and each apple sells for \(\$0.5\). How much money does John make?

\textbf{T:} How many trees does John plant?

\textbf{S:} Isn't it just \(3+4\) trees?

\textbf{T:} It resembles an array or multiplication. Could you try a different way?

\textbf{S:} \(3\times4=12\) trees.

\textbf{T:} If each of these \(12\) trees gives \(5\) apples, how many apples are there?

\textbf{S:} I don't know. Is it \(60\)? &
-- Correct

-- On track

-- Weak justification &
The answer is correct, but the student gives no reasoning, so the tutor should ask for justification rather than simply advance. \\

\midrule

\actionbadge{repairbg}{repairfg}{Arithmetic repair} &
\textbf{Problem:} John plants a plot of \(3\) trees by \(4\) trees. Each tree gives \(5\) apples, and each apple sells for \(\$0.5\). How much money does John make?

\textbf{T:} How many trees does John plant?

\textbf{S:} John plants \(3\times4=12\) trees.

\textbf{T:} That's right! Now, if each tree gives \(5\) apples, how many apples does John get in total?

\textbf{S:} Yes, I think John gets \(12\times5=70\) apples in total. &
-- Incorrect

-- On track

-- Arithmetic error &
The student selects the correct operation but miscomputes \(12\times5\), so the tutor should repair the arithmetic rather than change the subgoal. \\

\bottomrule
\end{tabularx}
\end{table}

\subsection{Pedagogical Action Space}
\label{app:action_space}

We define the six tutor actions (Table \ref{tab:action_space_definitions}) in \textsc{DICE-Bench} by combining classical views of tutoring as guided problem solving and scaffolded knowledge construction~\cite{simon1971human, fosnot1996constructivism, simon1995reconstructing, anderson1995cognitive, chi2001learning} with prior taxonomies of tutorial dialogue moves, Socratic questioning, feedback, hinting, and error remediation ~\cite{graesser2004autotutor, 10.5555/3737916.3740637, macina-etal-2023-mathdial}. The resulting action space is intentionally compact: it separates non-intervention from major intervention types while avoiding overly fine-grained labels that would be difficult to learn reliably.

\begin{table}[htbp]
\centering
\scriptsize
\caption{\footnotesize Pedagogical action space used for turn-level policy annotation in \textsc{DICE-Bench}.}
\label{tab:action_space_definitions}
\renewcommand{\arraystretch}{1.15}
\begin{tabular}{p{0.22\textwidth}p{0.72\textwidth}}
\toprule
\rowcolor{gray!25}
\textbf{Action} & \textbf{Meaning} \\
\midrule

\textbf{Accept and advance} &
The student has satisfied the current subgoal, so the tutor should acknowledge the response and move to the next step without adding extra scaffolding. \\

\textbf{Arithmetic repair} &
The student is following the right reasoning path but makes a numerical or computation error; the tutor should correct or guide the calculation while preserving the current subgoal. \\

\textbf{Conceptual hint} &
The student shows a misunderstanding of the underlying concept or operation; the tutor should provide a light conceptual cue without directly giving away the full solution. \\

\textbf{Handle question} &
The student asks a clarification or meta-level question; the tutor should answer the question and then return the student to the relevant subgoal. \\

\textbf{Ask for justification} &
The student's answer may be correct or partially correct, but the reasoning is missing or weak; the tutor should prompt the student to explain or justify their step. \\

\textbf{Refocus subgoal} &
The student skips the current subgoal, jumps ahead, or drifts from the intended trajectory; the tutor should redirect attention to the current step. \\

\bottomrule
\end{tabular}
\end{table}

\section{Intervention Value Calibration Details}
\label{app:iv_calibration}
\subsection{Detailed IV Audit Aggregations}
\label{app:iv_audit_details}

For action-level auditing, we compute the action-conditional expected IV,
\[
\overline{\mathrm{IV}}(a)
=
\mathbb{E}_{x:\,a_i=a}
\left[
\widehat{\mathrm{IV}}(x,a)
\right],
\qquad a\in\mathcal{A}^{+}.
\]
Actions with consistently high \(\overline{\mathrm{IV}}\) correspond to robust interventions, while actions with \(\overline{\mathrm{IV}}\approx 0\) indicate possible over-prescription.

For source-level auditing, we compare IV across dialogue-derived and step-level annotations. Let \(\mathcal{D}_{\mathrm{dial}}\) denote original multi-turn tutoring states and \(\mathcal{D}_{\mathrm{step}}\) denote synthetic step-level states. Information asymmetry predicts
\[
\mathbb{E}_{x\sim\mathcal{D}_{\mathrm{dial}}}
\left[
\widehat{\mathrm{IV}}(x,a_i)
\right]
<
\mathbb{E}_{x\sim\mathcal{D}_{\mathrm{step}}}
\left[
\widehat{\mathrm{IV}}(x,a_i)
\right].
\]
Because \(\mathcal{D}_{\mathrm{step}}\) contains both \(q^{-}\) and \(q^{+}\) variants of the same source problems, a persistent gap between \(\mathcal{D}_{\mathrm{dial}}\) and \(\mathcal{D}_{\mathrm{step}}\) is less likely to be explained by problem difficulty alone and more consistent with annotation-source effects.

\subsection{IV-Calibrated Policy Optimization Details}
\label{app:iv_policy}

The weighting rule in Eq.~\ref{eq:weights} treats non-intervention labels as fully reliable because they correspond to the tutor explicitly choosing not to intervene. Active interventions are trusted only in proportion to their estimated marginal value over \(a_0\). High-IV interventions receive larger gradient weight, near-zero-IV interventions contribute little, and negative-IV interventions are clipped to zero so that harmful supervision cannot directly reinforce the policy.

The KL term in Eq.~\ref{eq:dice} regularizes the IV-calibrated policy toward the original supervised policy \(\pi_{\textsc{sft}}\). This is important because IV is estimated using a simulated student and an LLM judge, both of which introduce noise. The hyperparameter \(\beta\geq0\) controls the strength of this regularization: \(\beta=0\) corresponds to pure IV-weighted training, while larger \(\beta\) keeps the policy closer to SFT.

Eq.~\ref{eq:dice} can be viewed as a form of advantage-weighted regression~\cite{peng2019advantageweightedregressionsimplescalable}, where the advantage signal is the counterfactual Intervention Value rather than a learned value function. The KL penalty plays the role of a trust region, preventing the policy from moving too far from supervised behavior. Unlike standard reinforcement learning, the signal is not an absolute return; it is a direct comparison between intervention and non-intervention, making it interpretable as a measure of intervention necessity.

\subsection{Action Selection and Generation Details}
\label{app:generation_details}

\paragraph{MAP action selection.}
A naive policy selects the maximum-probability action,
\begin{equation}
\label{eq:dice_map}
    \hat{a}_t^{\textsc{map}}
    =
    \arg\max_{a\in\mathcal{A}}
    \pi_\theta(a\mid x_t).
\end{equation}
This treats intervention and non-intervention symmetrically, even though intervention carries pedagogical and computational cost.

\paragraph{Cost-sensitive intervention margin.}
The cost-sensitive rule in Eq.~\ref{eq:dice_cost} is equivalent to intervening only when the best active action exceeds non-intervention by a log-probability margin:
\begin{equation}
\label{eq:dice_margin}
    \max_{a\in\mathcal{A}^{+}}
    \left[
        \log \pi_\theta(a\mid x_t)
        -
        \log \pi_\theta(a_0\mid x_t)
    \right]
    >
    \lambda.
\end{equation}
This form shows that \(\lambda\) acts as a log-likelihood-ratio threshold: the tutor intervenes only when there is sufficiently stronger evidence for an active pedagogical action than for letting the student continue.

\paragraph{Thresholded intervention ablation.}
We also evaluate a simpler binary intervention rule. Define
\begin{equation}
\label{eq:p_intervene}
    p_{\mathrm{int}}(x_t)
    =
    \sum_{a\in\mathcal{A}^{+}}\pi_\theta(a\mid x_t)
    =
    1-\pi_\theta(a_0\mid x_t).
\end{equation}
Given threshold \(\tau\in[0,1]\), the policy chooses
\begin{equation}
\label{eq:dice_threshold}
    \hat{a}_t^{\mathrm{thr}}
    =
    \begin{cases}
        a_0, & \text{if } p_{\mathrm{int}}(x_t)\leq \tau,\\[3pt]
        \arg\max_{a\in\mathcal{A}^{+}}\pi_\theta(a\mid x_t),
        & \text{otherwise.}
    \end{cases}
\end{equation}
Unlike \textsc{DICE-Cost}, which compares each active action directly against \(a_0\), this ablation first makes a binary intervention decision and then selects the most likely active action.

\paragraph{Decision--realization factorization.}
The two-stage model can be written as
\begin{equation}
\label{eq:factorization}
    p(u_t\mid x_t)
    =
    \sum_{a\in\mathcal{A}}
    p_\phi(u_t\mid x_t,a)\,
    \pi_\theta(a\mid x_t).
\end{equation}
At inference, \textsc{DICE} does not marginalize over all actions. It selects a calibrated action using Eq.~\ref{eq:dice_cost} and generates only from the corresponding action-conditioned distribution. This keeps the policy decision auditable and makes intervention behavior controllable through \(\lambda\).

\section{Theoretical Analysis}
\label{appendix:theory}

We analyze \textsc{DICE} from a decision-theoretic perspective. We formalize why cost-sensitive inference induces calibrated intervention (\S\ref{appendix:cost_sensitive}), why an oracle adaptive intervention policy dominates uniform intervention under the same budget (\S\ref{appendix:oracle_adaptive}), how the IV-weighted KL objective modifies the supervised policy (\S\ref{appendix:iv_kl_theory}), and how provenance-dependent IV gaps can be used as a diagnostic for annotation effects (\S\ref{appendix:provenance_theory}). Throughout, \(a_0\) denotes the non-intervention action and \(\mathcal{A}^{+}=\mathcal{A}\setminus\{a_0\}\) denotes the set of active interventions.

\paragraph{Notation and basic properties.}
For a tutoring state \(x\) and action \(a\in\mathcal{A}\), let
\[
\mu(x,a)
=
\mathbb{E}_{y\sim p_{\mathrm{student}}(\cdot\mid x,T(a,x))}
\left[
r(x,a,y)
\right]
\]
denote expected student progress after the tutor realizes action \(a\) as teacher utterance \(T(a,x)\). The progress score satisfies \(r\in[0,1]\). The Intervention Value is
\[
\mathrm{IV}(x,a)
=
\mu(x,a)-\mu(x,a_0).
\]
Thus, \(\mathrm{IV}(x,a_0)=0\), \(\mathrm{IV}(x,a)\in[-1,1]\), and the sign of IV directly characterizes marginal intervention utility: \(\mathrm{IV}>0\) indicates that intervention improves expected progress over non-intervention, \(\mathrm{IV}\approx0\) indicates little marginal benefit, and \(\mathrm{IV}<0\) indicates that intervention is worse than non-intervention. For active interventions, the IV-derived training weight
\[
w(x,a)=\max(\mathrm{IV}(x,a),0),\qquad a\in\mathcal{A}^{+},
\]
suppresses harmful or unnecessary intervention labels, while non-intervention labels use \(w(x,a_0)=1\).

\subsection{Cost-Sensitive Intervention Selection}
\label{appendix:cost_sensitive}

At inference time, \textsc{DICE} selects actions by maximizing a cost-sensitive score
\[
S_\lambda(a,x)
=
\log \pi_\theta(a\mid x)
-
\lambda\,\mathbf{1}[a\neq a_0],
\]
where \(\lambda\geq0\) is the intervention penalty. This rule has a simple likelihood-ratio interpretation.

\begin{proposition}[Cost-sensitive inference as a likelihood-ratio test]
\label{prop:cost_margin}
Let
\[
a^{+}_{\theta}(x)
=
\arg\max_{a\in\mathcal{A}^{+}}
\pi_\theta(a\mid x)
\]
be the most likely active intervention under the policy. The cost-sensitive rule
\[
\hat{a}_\lambda(x)
=
\arg\max_{a\in\mathcal{A}}
S_\lambda(a,x)
\]
selects an active intervention if and only if
\begin{equation}
\label{eq:margin_test}
    \frac{
    \pi_\theta(a^{+}_{\theta}(x)\mid x)
    }{
    \pi_\theta(a_0\mid x)
    }
    >
    e^{\lambda}.
\end{equation}
\end{proposition}

\begin{proof}
The non-intervention score is
\[
S_\lambda(a_0,x)=\log\pi_\theta(a_0\mid x),
\]
since \(a_0\) incurs no intervention penalty. The best active-intervention score is
\[
\max_{a\in\mathcal{A}^{+}}S_\lambda(a,x)
=
\log\pi_\theta(a^{+}_{\theta}(x)\mid x)-\lambda.
\]
An active intervention is selected if and only if
\[
\log\pi_\theta(a^{+}_{\theta}(x)\mid x)-\lambda
>
\log\pi_\theta(a_0\mid x).
\]
Rearranging gives
\[
\log\frac{
\pi_\theta(a^{+}_{\theta}(x)\mid x)
}{
\pi_\theta(a_0\mid x)
}
>
\lambda.
\]
Exponentiating both sides yields Eq.~\ref{eq:margin_test}. The proof is complete.
\end{proof}

Thus, \(\lambda\) is not merely a decoding hyperparameter. It is a log-likelihood-ratio threshold: the tutor intervenes only when the best active action is at least \(e^\lambda\) times more likely than non-intervention under the learned policy.

\begin{proposition}[Monotonicity of intervention rate]
\label{prop:lambda_monotone}
For a fixed policy \(\pi_\theta\) and a fixed distribution over tutoring states, the intervention rate under \textsc{DICE-Cost} is non-increasing in \(\lambda\).
\end{proposition}

\begin{proof}
For each state \(x\), define the intervention margin
\[
m_\theta(x)
=
\log\pi_\theta(a^{+}_{\theta}(x)\mid x)
-
\log\pi_\theta(a_0\mid x).
\]
By Proposition~\ref{prop:cost_margin}, the policy intervenes at \(x\) if and only if \(m_\theta(x)>\lambda\). For \(\lambda_1\leq\lambda_2\),
\[
\{x:m_\theta(x)>\lambda_2\}
\subseteq
\{x:m_\theta(x)>\lambda_1\}.
\]
Therefore,
\[
\mathbb{E}_x[\mathbf{1}[m_\theta(x)>\lambda]]
\]
is non-increasing in \(\lambda\). The proof is complete.
\end{proof}

This proposition justifies tuning \(\lambda\) on validation data: increasing \(\lambda\) makes the tutor more selective, producing a monotone intervention-frequency curve.

\subsection{Oracle Adaptive Intervention Dominates Uniform Intervention}
\label{appendix:oracle_adaptive}

We next formalize why per-state intervention decisions can improve over uniform intervention strategies. This result is stated for an oracle adaptive policy that has access to the true intervention values. \textsc{DICE} should be interpreted as a learned approximation to this oracle policy.

\begin{definition}[Binary intervention utility]
\label{def:binary_iv}
For a tutoring state \(x\), define the best active intervention value
\[
V(x)
=
\max_{a\in\mathcal{A}^{+}}
\mathrm{IV}(x,a).
\]
A state is intervention-beneficial if \(V(x)>0\), and non-beneficial otherwise.
\end{definition}

\begin{definition}[Intervention distortion]
\label{def:intervention_distortion}
For a binary intervention decision \(d\in\{0,1\}\), where \(d=1\) means the tutor intervenes and \(d=0\) means the tutor does not intervene, define the intervention distortion
\begin{equation}
\label{eq:intervention_distortion}
\mathcal{D}(d,x)
=
\begin{cases}
V(x), & d=0 \text{ and } V(x)>0,\\[3pt]
- V(x), & d=1 \text{ and } V(x)\leq0,\\[3pt]
0, & \text{otherwise}.
\end{cases}
\end{equation}
\end{definition}

This distortion captures two errors: under-intervention, where the tutor does not intervene even though the best active intervention has positive value, and over-intervention, where the tutor intervenes despite non-positive intervention value.

\begin{definition}[Uniform and oracle adaptive policies]
\label{def:uniform_adaptive}
A uniform intervention policy intervenes with fixed probability \(\rho\in[0,1]\), independent of \(x\). An oracle adaptive policy selects an input-dependent decision \(d_{\mathrm{oracle}}(x)\in\{0,1\}\) under the same average budget:
\[
\mathbb{E}_x[d_{\mathrm{oracle}}(x)] = \rho.
\]
Among all policies satisfying this budget, the oracle adaptive policy selects the \(\rho\)-fraction of states with the largest values of \(V(x)\).
\end{definition}

\begin{theorem}[Oracle adaptive intervention dominates uniform intervention]
\label{thm:adaptive_dominates}
Under the same average intervention budget \(\rho\), the oracle adaptive policy achieves no larger expected distortion than the uniform intervention policy:
\begin{equation}
\label{eq:adaptive_dominance}
\mathbb{E}_x
\left[
\mathcal{D}(d_{\mathrm{oracle}}(x),x)
\right]
\leq
\mathbb{E}_x
\left[
\mathcal{D}(d_{\mathrm{unif}}(x),x)
\right].
\end{equation}
\end{theorem}

\begin{proof}
Let \(V(x)\) be the best active intervention value. For any binary decision \(d\), the distortion in Eq.~\ref{eq:intervention_distortion} can be written as
\[
\mathcal{D}(d,x)
=
\max(V(x),0)
-
d\,V(x).
\]
To verify this identity, consider two cases. If \(V(x)>0\), then \(\max(V(x),0)-dV(x)=V(x)\) when \(d=0\), and \(0\) when \(d=1\). If \(V(x)\leq0\), then \(\max(V(x),0)-dV(x)=0\) when \(d=0\), and \(-V(x)\) when \(d=1\). This matches Eq.~\ref{eq:intervention_distortion}.

Taking expectation,
\[
\mathbb{E}_x[\mathcal{D}(d,x)]
=
\mathbb{E}_x[\max(V(x),0)]
-
\mathbb{E}_x[d(x)V(x)].
\]
The first term is independent of the intervention policy. Therefore, minimizing expected distortion under the budget constraint \(\mathbb{E}_x[d(x)]=\rho\) is equivalent to maximizing
\[
\mathbb{E}_x[d(x)V(x)].
\]
Among all binary policies with average intervention budget \(\rho\), this quantity is maximized by assigning \(d(x)=1\) to the \(\rho\)-fraction of states with the largest values of \(V(x)\). This is exactly the oracle adaptive policy by Definition~\ref{def:uniform_adaptive}. Hence no uniform policy with the same average budget can achieve lower expected distortion. The proof is complete.
\end{proof}

\begin{remark}
This theorem formalizes the advantage of intervention calibration. A response-generation tutor that always speaks corresponds to a uniform intervention strategy with \(\rho=1\). More generally, a fixed-rate tutor allocates intervention without considering the state-specific value \(V(x)\). In contrast, \textsc{DICE} learns a state-dependent policy that approximates the oracle ranking induced by \(V(x)\), concentrating interventions on states where they are expected to improve student progress.
\end{remark}

\subsection{IV-Weighted KL Policy Optimization}
\label{appendix:iv_kl_theory}

We characterize the optimal policy under the \textsc{DICE} training objective. For a single tutoring state \(x\) with annotated action \(a^\star\) and IV-derived weight \(w\geq0\), consider the state-wise objective
\[
\mathcal{L}_{x}(\pi)
=
-w\log\pi(a^\star)
+
\beta
D_{\mathrm{KL}}
\left(
\pi(\cdot)
\,\middle\|\,
\pi_{\textsc{sft}}(\cdot\mid x)
\right),
\]
where \(\pi\) ranges over distributions on \(\mathcal{A}\), and \(\beta>0\). Let \(q(a)=\pi_{\textsc{sft}}(a\mid x)\).

\begin{proposition}[Fixed-point characterization of the IV-KL minimizer]
\label{prop:kl_conservative}
Assume \(q(a)>0\) for all \(a\in\mathcal{A}\). The unique minimizer \(\pi^\star\) of \(\mathcal{L}_{x}(\pi)\) satisfies
\begin{equation}
\label{eq:implicit_solution}
\pi^\star(a)
\propto
q(a)
\exp
\left(
\frac{
w\,\mathbf{1}[a=a^\star]
}{
\beta\,\pi^\star(a^\star)
}
\right).
\end{equation}
This is an implicit fixed-point equation in \(\pi^\star(a^\star)\). It has the following qualitative consequences:
\begin{enumerate}
    \item If \(w=0\), then \(\pi^\star=q\).
    \item If \(w>0\), then \(\pi^\star(a^\star)>q(a^\star)\).
    \item If \(w>0\), then probability mass is shifted from non-annotated actions toward \(a^\star\).
\end{enumerate}
\end{proposition}

\begin{proof}
We minimize
\[
-w\log\pi(a^\star)
+
\beta
\sum_{a\in\mathcal{A}}
\pi(a)\log\frac{\pi(a)}{q(a)}
\]
subject to \(\sum_{a}\pi(a)=1\). The Lagrangian is
\[
\mathcal{J}(\pi,\eta)
=
-w\log\pi(a^\star)
+
\beta
\sum_{a\in\mathcal{A}}
\pi(a)\log\frac{\pi(a)}{q(a)}
+
\eta\left(\sum_{a\in\mathcal{A}}\pi(a)-1\right).
\]
For any \(a\in\mathcal{A}\), stationarity gives
\[
\frac{\partial \mathcal{J}}{\partial \pi(a)}
=
-\frac{w\,\mathbf{1}[a=a^\star]}{\pi(a^\star)}
+
\beta
\left(
\log\frac{\pi(a)}{q(a)}+1
\right)
+
\eta
=
0.
\]
Rearranging,
\[
\log\frac{\pi(a)}{q(a)}
=
\frac{
w\,\mathbf{1}[a=a^\star]
}{
\beta\,\pi(a^\star)
}
-
1
-
\frac{\eta}{\beta}.
\]
Exponentiating and absorbing all terms independent of \(a\) into the normalization constant yields Eq.~\ref{eq:implicit_solution}. The objective is strictly convex on the probability simplex when \(\beta>0\) and \(q(a)>0\) for all \(a\), so the minimizer is unique.

If \(w=0\), the exponential factor is the same for every action, so \(\pi^\star=q\). If \(w>0\), the exponential factor for \(a^\star\) is strictly larger than the factor for any \(a\neq a^\star\), so the ratio \(\pi^\star(a^\star)/q(a^\star)\) is larger than \(\pi^\star(a)/q(a)\) for all \(a\neq a^\star\). Since probabilities sum to one, this implies \(\pi^\star(a^\star)>q(a^\star)\), with mass redistributed from other actions. The proof is complete.
\end{proof}

\begin{corollary}[Effect of the KL coefficient]
\label{cor:beta_effect}
For fixed \(w>0\), \(\beta\) controls the strength of IV-based policy correction:
\begin{enumerate}
    \item As \(\beta\to\infty\), \(\pi^\star\to\pi_{\textsc{sft}}(\cdot\mid x)\).
    \item As \(\beta\to0^{+}\), \(\pi^\star\) concentrates on \(a^\star\).
\end{enumerate}
\end{corollary}

\begin{proof}
When \(\beta\to\infty\), the KL term dominates and the minimizer approaches the unique minimizer of
\[
D_{\mathrm{KL}}(\pi\|\pi_{\textsc{sft}}(\cdot\mid x)),
\]
namely \(\pi_{\textsc{sft}}(\cdot\mid x)\). When \(\beta\to0^{+}\), the supervised term \(-w\log\pi(a^\star)\) dominates. This term is minimized by maximizing \(\pi(a^\star)\), whose limiting solution places all probability mass on \(a^\star\). The proof is complete.
\end{proof}

\begin{remark}
Proposition~\ref{prop:kl_conservative} explains the role of IV-KL training. High-IV labels increase the probability of the annotated action, while zero-IV or negative-IV labels do not move the policy away from the SFT reference. The KL coefficient \(\beta\) controls how aggressively the policy trusts the IV signal.
\end{remark}

\subsection{Provenance-Dependent Intervention Value}
\label{appendix:provenance_theory}

Finally, we formalize how IV can be used to diagnose annotation-source effects. This section is not a causal proof that provenance alone determines IV. Rather, it gives a testable diagnostic: if IV differs across annotation sources even when difficulty trends do not explain the ordering, then the gap is more consistent with annotation-source effects than with problem difficulty alone.

\begin{definition}[Annotation sources]
\label{def:annotation_sources}
Let \(\mathcal{D}_{\mathrm{dial}}\) denote dialogue-derived source-problem annotations, extracted from existing multi-turn tutoring dialogues. Let \(\mathcal{D}_{\mathrm{step}}\) denote step-variant annotations, constructed from local subgoal-level student states for easy and hard variants. We further write
\[
\mathcal{D}_{\mathrm{step}}
=
\mathcal{D}_{\mathrm{easy}}
\cup
\mathcal{D}_{\mathrm{hard}}.
\]
\end{definition}

\begin{definition}[Over-intervention rate]
\label{def:over_intervention_rate}
For a dataset \(\mathcal{D}\) of annotated decision points, define the over-intervention rate as
\[
\mathrm{OI}(\mathcal{D})
=
\frac{
\left|
\{(x,a)\in\mathcal{D}: a\in\mathcal{A}^{+},\ \mathrm{IV}(x,a)\leq0\}
\right|
}{
\left|
\{(x,a)\in\mathcal{D}: a\in\mathcal{A}^{+}\}
\right|
}.
\]
This measures the fraction of prescribed active interventions that provide no positive marginal value over non-intervention.
\end{definition}

\begin{hypothesis}[Provenance-dependent IV gap]
\label{hyp:provenance_gap}
If dialogue-derived annotations over-prescribe intervention relative to local step-variant annotations, then
\begin{equation}
\label{eq:provenance_iv_gap}
\mathbb{E}_{(x,a)\sim\mathcal{D}_{\mathrm{dial}}}
\left[
\mathrm{IV}(x,a)
\right]
<
\mathbb{E}_{(x,a)\sim\mathcal{D}_{\mathrm{step}}}
\left[
\mathrm{IV}(x,a)
\right],
\end{equation}
and
\begin{equation}
\label{eq:provenance_oi_gap}
\mathrm{OI}(\mathcal{D}_{\mathrm{dial}})
>
\mathrm{OI}(\mathcal{D}_{\mathrm{step}}).
\end{equation}
\end{hypothesis}

\begin{proposition}[Difficulty-only explanations imply a monotone difficulty trend]
\label{prop:difficulty_diagnostic}
Suppose that the IV gap between \(\mathcal{D}_{\mathrm{dial}}\) and \(\mathcal{D}_{\mathrm{step}}\) is explained only by problem difficulty. Then the expected IV should vary monotonically with the difficulty ladder from easy variants to source problems to hard variants. Therefore, if the empirical ordering is non-monotone with respect to difficulty, and both easy and hard step-variant annotations have higher IV than dialogue-derived source-problem annotations, then difficulty alone is insufficient to explain the provenance gap.
\end{proposition}

\begin{proof}
Under a difficulty-only explanation, annotation source has no independent effect on IV once difficulty is fixed. Therefore, expected IV should be expressible as a function of difficulty alone. Since the data are arranged as a difficulty ladder consisting of easy variants, source problems, and hard variants, a difficulty-only explanation predicts a monotone trend across these three groups, either increasing or decreasing depending on how difficulty affects student recoverability.

However, if both
\[
\mathbb{E}_{\mathcal{D}_{\mathrm{easy}}}[\mathrm{IV}]
>
\mathbb{E}_{\mathcal{D}_{\mathrm{dial}}}[\mathrm{IV}]
\]
and
\[
\mathbb{E}_{\mathcal{D}_{\mathrm{hard}}}[\mathrm{IV}]
>
\mathbb{E}_{\mathcal{D}_{\mathrm{dial}}}[\mathrm{IV}]
\]
hold, then the source-problem dialogue-derived group lies below both the easier and harder step-variant groups. This pattern is not monotone along the difficulty ladder. Hence, difficulty alone is insufficient to explain the observed IV gap. The remaining difference is more consistent with annotation-source effects, such as deriving labels from existing teacher-led dialogues rather than local subgoal-level decision points. The proof is complete.
\end{proof}

\begin{remark}[Empirical diagnostic]
In our experiments, the IV ordering follows
\[
\mathrm{IV}_{\mathrm{hard}}
>
\mathrm{IV}_{\mathrm{easy}}
>
\mathrm{IV}_{\mathrm{dial}},
\]
indicating that both easy and hard step-variant annotations have higher marginal intervention value than dialogue-derived source-problem annotations. This supports the interpretation that the low IV of dialogue-derived labels is not explained by difficulty alone, and motivates IV-calibrated policy learning.
\end{remark}

\section{Detailed Experiments}
\label{app:experimental_details}

\subsection{Policy Model Training}
\label{app:policy_training}

We train the \textsc{DICE} policy as an action-selection model. Each example encodes the tutoring state, and the model predicts the next pedagogical action. To avoid leakage, train, validation, and test splits are created by grouping examples by source problem ID, so all states derived from the same source problem and its variants remain in the same split. The resulting model configuration and training hyperparameters are summarized in Tables~\ref{tab:policy_training_config} and~\ref{tab:policy_lora_hparams}.

\begin{table}[htbp]
\centering
\begin{minipage}[t]{0.48\textwidth}
\centering
\scriptsize
\caption{\footnotesize Policy model configuration.}
\label{tab:policy_training_config}
\renewcommand{\arraystretch}{1.12}
\begin{tabularx}{\linewidth}{>{\raggedright\arraybackslash}p{0.38\linewidth} >{\raggedright\arraybackslash}X}
\toprule
\rowcolor{gray!25}
\textbf{Component} & \textbf{Setting} \\
\midrule
Backbone model & Qwen-2.5-7B-Instruct \\
Training objective & Causal language modeling over action outputs \\
Policy output format & Action-only prediction \\
Action space & \(\{\)accept, arithmetic repair, handle question, refocus, conceptual hint, justify\(\}\) \\
Input context & Problem, reference analysis, answer, current subgoal, teacher question, dialogue history, and student response \\
Split strategy & Grouped by source problem ID \\
Split ratio & 80/10/10 train, validation, test \\
Maximum sequence length & 2048 tokens \\
Seed & 42 \\
\bottomrule
\end{tabularx}
\end{minipage}
\hfill
\begin{minipage}[t]{0.48\textwidth}
\centering
\scriptsize
\caption{\footnotesize LoRA and optimization hyperparameters.}
\label{tab:policy_lora_hparams}
\renewcommand{\arraystretch}{1.12}
\begin{tabularx}{\linewidth}{>{\raggedright\arraybackslash}p{0.48\linewidth} >{\raggedright\arraybackslash}X}
\toprule
\rowcolor{gray!25}
\textbf{Hyperparameter} & \textbf{Value} \\
\midrule
LoRA rank \(r\) & 16 \\
LoRA \(\alpha\) & 32 \\
LoRA dropout & 0.05 \\
Target modules & \(q\), \(k\), \(v\), \(o\), up, down, gate projections \\
Training epochs & 3 \\
Learning rate & \(2\times10^{-4}\) \\
Weight decay & 0.0 \\
Warmup ratio & 0.03 \\
Per-device train batch size & 2 \\
Per-device eval batch size & 2 \\
Gradient accumulation steps & 8 \\
Precision & BF16 \\
Quantization & 4-bit NF4 with double quantization \\
Gradient checkpointing & Enabled \\
KL coefficient \(\beta\) & 0.05 \\
\bottomrule
\end{tabularx}
\end{minipage}
\end{table}

\subsection{\textsc{DICE} Training and Inference}
\label{app:dice_algorithm}

The full \textsc{DICE} procedure is summarized in Algorithm~\ref{alg:dice}. Starting from turn-level action annotations, we first train a supervised action policy, estimate Intervention Value through controlled rollouts, and then train an IV-calibrated policy with KL regularization to the supervised policy. At inference time, the policy selects a discrete pedagogical action before response generation, allowing the tutor to explicitly decide whether intervention is necessary.

\begin{algorithm}[htbp]
\caption{\textsc{DICE}: IV-Calibrated Tutoring Policy}
\label{alg:dice}
\small
\begin{algorithmic}[1]
\Require Annotated tutoring states \(\mathcal{D}=\{(x_i,a_i)\}_{i=1}^{N}\); action set \(\mathcal{A}\); non-intervention action \(a_0\); student simulator \(p_{\mathrm{student}}\); judge model \(\mathcal{J}\); KL weight \(\beta\)
\Ensure Calibrated tutoring policy \(\pi_{\theta}\)

\vspace{0.3em}
\Statex \textbf{/* Stage 1: Supervised action policy */}
\State Train \(\pi_{\mathrm{SFT}}\) on \(\mathcal{D}\) using standard action prediction.

\vspace{0.3em}
\Statex \textbf{/* Stage 2: Intervention Value estimation */}
\For{each example \((x_i,a_i)\in\mathcal{D}\)}
    \If{\(a_i = a_0\)}
        \State Set \(w_i \leftarrow 1\)
    \Else
        \State Render controlled utterances \(T(a_i,x_i)\) and \(T(a_0,x_i)\)
        \State Roll out student responses under both actions using \(p_{\mathrm{student}}\)
        \State Score rollout progress with judge \(\mathcal{J}\)
        \State Estimate \(\widehat{\mathrm{IV}}(x_i,a_i)=\hat{\mu}(x_i,a_i)-\hat{\mu}(x_i,a_0)\)
        \State Set \(w_i \leftarrow \max(\widehat{\mathrm{IV}}(x_i,a_i),0)\)
    \EndIf
\EndFor

\vspace{0.3em}
\Statex \textbf{/* Stage 3: IV-calibrated policy optimization */}
\State Train \(\pi_{\theta}\) with IV-weighted supervision:
\[
\mathcal{L}_{\textsc{DICE}}
=
-\frac{1}{\sum_i w_i}\sum_{i=1}^{N}w_i\log\pi_{\theta}(a_i\mid x_i)
+
\beta D_{\mathrm{KL}}\!\left(\pi_{\theta}(\cdot\mid x_i)\,\|\,\pi_{\mathrm{SFT}}(\cdot\mid x_i)\right).
\]

\vspace{0.3em}
\Statex \textbf{/* Stage 4: Action-conditioned tutoring */}
\For{each tutoring turn \(t\)}
    \State Observe tutoring state \(x_t\)
    \State Select action \(\hat{a}_t\) using MAP, Cost, or Threshold inference
    \If{\(\hat{a}_t = a_0\)}
        \State Accept and advance without generating active guidance
    \Else
        \State Generate teacher utterance \(u_t\sim p_{\phi}(\cdot\mid x_t,\hat{a}_t)\)
    \EndIf
\EndFor
\State \Return \(\pi_{\theta}\)
\end{algorithmic}
\end{algorithm}

\subsection{End-to-End Tutoring Evaluation}
\label{app:end_to_end_eval}

We evaluate end-to-end tutoring in two settings: a single-problem setting, where the tutor guides the student on one variant at a time, and a session-level setting, where the tutor guides the same student through a problem triplet \((q^{-},q,q^{+})\) under a shared turn budget. The single-problem setting uses a maximum of 8 teacher-student turns per problem, while the session-level setting uses a total budget of 20 turns across the full triplet. The inference configurations for both settings are summarized in Tables~\ref{tab:single_problem_inference_config} and~\ref{tab:session_inference_config}.

\begin{table}[htbp]
\centering
\begin{minipage}[t]{0.48\textwidth}
\centering
\scriptsize
\caption{\footnotesize Single-problem tutoring inference setup.}
\label{tab:single_problem_inference_config}
\renewcommand{\arraystretch}{1.12}
\begin{tabularx}{\linewidth}{>{\raggedright\arraybackslash}p{0.43\linewidth} >{\raggedright\arraybackslash}X}
\toprule
\rowcolor{gray!25}
\textbf{Component} & \textbf{Setting} \\
\midrule
Evaluation unit & One problem variant at a time \\
Variants & \(q^{-}\), \(q\), and \(q^{+}\) evaluated independently \\
Maximum turns & 8 teacher-student turns \\
Teacher backbone & Qwen-2.5-7B-Instruct \\
Student simulator & Qwen-2.5-1.5B-Instruct \\
Policy backbone & Qwen-2.5-7B-Instruct with LoRA adapter \\
Policy scoring & Constrained action-label likelihood \\
Score normalization & Mean token log-likelihood \\
Teacher decoding & Temperature 0.2, top-\(p=0.95\) \\
Student decoding & Temperature 0.6, top-\(p=0.95\) \\
Teacher max tokens & 256 \\
Student max tokens & 256 \\
Context length & 2048 tokens \\
Total Problems & 750 \\
\bottomrule
\end{tabularx}
\end{minipage}
\hfill
\begin{minipage}[t]{0.48\textwidth}
\centering
\scriptsize
\caption{\footnotesize Session-level tutoring inference setup.}
\label{tab:session_inference_config}
\renewcommand{\arraystretch}{1.12}
\begin{tabularx}{\linewidth}{>{\raggedright\arraybackslash}p{0.43\linewidth} >{\raggedright\arraybackslash}X}
\toprule
\rowcolor{gray!25}
\textbf{Component} & \textbf{Setting} \\
\midrule
Evaluation unit & One full problem triplet \((q^{-},q,q^{+})\) \\
Progression order & \(q^{-}\rightarrow q\rightarrow q^{+}\) \\
Maximum turns & 20 total turns per session \\
Teacher backbone & Qwen-2.5-7B-Instruct \\
Student simulator & Qwen-2.5-1.5B-Instruct \\
Policy backbone & Qwen-2.5-7B-Instruct with LoRA adapter \\
Policy scoring & Constrained action-label likelihood \\
Score normalization & Mean token log-likelihood \\
Teacher decoding & Temperature 0.2, top-\(p=0.95\) \\
Student decoding & Temperature 0.6, top-\(p=0.95\) \\
Teacher max tokens & 256 \\
Student max tokens & 256 \\
Context length & 4096 tokens \\
Total Problems & 750 \\
\bottomrule
\end{tabularx}
\end{minipage}
\end{table}

\subsection{Evaluation Metrics}
\label{app:evaluation_metrics}

For turn-level policy prediction, we report action accuracy and macro-F1 over the six pedagogical actions. We also compute binary intervention metrics by collapsing the action space into non-intervention \(a_0\) and active intervention \(\mathcal{A}^{+}=\mathcal{A}\setminus\{a_0\}\). For tutoring, we follow interactive evaluation protocols used in prior tutoring work such as SocraticLM~\cite{10.5555/3737916.3740637} and MathDial~\cite{macina-etal-2023-mathdial}, and additionally track intervention-specific failures.

Let \(c_t\in\{0,1\}\) denote whether the student's latest response is correct at turn \(t\), and let \(\hat{a}_t\) be the tutor's selected action. We define over-intervention as intervening when the student is already correct:
\[
    \mathrm{OverInt}
    =
    \frac{1}{T}
    \sum_{t=1}^{T}
    \mathbb{I}\!\left[c_t=1 \wedge \hat{a}_t\in\mathcal{A}^{+}\right].
\]
We define false acceptance as accepting an incorrect student response:
\[
    \mathrm{FalseAcc}
    =
    \frac{1}{T}
    \sum_{t=1}^{T}
    \mathbb{I}\!\left[c_t=0 \wedge \hat{a}_t=a_0\right].
\]
For session-level evaluation, each session contains the triplet \((q^{-},q,q^{+})\) under a shared turn budget. We report all-three success, variant-level success, average variants solved per session, average turns used, over-intervention, and unsafe final-answer acceptance. Unsafe final-answer acceptance occurs when the tutor accepts a final-answer-like student response that is incorrect.

\noindent\textbf{Acceptance verifier and final-answer guard.}
The acceptance verifier and final-answer guard are part of the session-level \textsc{DICE} framework, not a post-hoc evaluation correction. In single-problem tutoring, we observed that learning a selective policy substantially reduces unnecessary interventions, but does not by itself eliminate a different failure mode: the tutor may still accept a response that only appears correct. This is expected because the teacher is a standard instruction-tuned model rather than an oracle or frontier-grade mathematical verifier, and its pretrained knowledge alone is not always sufficient to decide when a student response should be trusted.

To address this, \textsc{DICE} adds a safety check only when the policy selects the non-intervention action \(a_0=\mathrm{accept\_and\_advance}\). The verifier receives the problem, reference answer, solution context, dialogue history, and latest student response, and predicts whether the response is \(\mathrm{CORRECT}\), \(\mathrm{INCORRECT}\), or \(\mathrm{UNCLEAR}\). If the response is not verified as \(\mathrm{CORRECT}\), \textsc{DICE} blocks the accept decision and replaces it with \(\mathrm{ask\_for\_justification}\). We further add a final-answer guard for the most severe case: if a response appears to be a final answer but fails exact answer matching, it is treated as \(\mathrm{UNCLEAR}\) even when the verifier is overly permissive. Together, these modules act as safety nets around accept decisions, preventing the tutor from prematurely validating incorrect student answers while preserving the policy's ability to intervene less often.

\subsection{Compute Cost}
\label{app:compute_cost}

Policy training was performed on a single A100 GPU with 80GB memory. Training the default Qwen-2.5-7B policy for three epochs took approximately \(4\) hours; across all policy training runs, the average training time was \(3.55\) hours. Additional training hyperparameters are reported in Table~\ref{tab:policy_training_config}.

Inference cost varies substantially with GPU availability and model-loading configuration. For the subset used in our reported tutoring experiments, we used either a \(4\times\)A5000 setup with 24GB memory per GPU or a single A100 GPU with 80GB memory. GPU memory utilization was set close to the maximum feasible level for each run, approximately 80GB on the A100 setup. Each end-to-end tutoring inference run took approximately \(2.5\) hours on average. 

\noindent\textbf{Inference efficiency.}
At inference time, \textsc{DICE} adds a decision-layer overhead: each turn requires a policy call, and accept decisions may additionally trigger the verifier. However, this overhead is partly offset by shorter tutoring trajectories, since calibrated action selection reduces unnecessary teacher-student turns. This is reflected in the session results, where \textsc{DICE-Cost} achieves the highest all-three success rate while using fewer turns than all response-only tutoring baselines. Thus, \textsc{DICE} can improve interaction-time efficiency when the cost of extra policy/verifier calls is outweighed by fewer expensive tutor generations.

\section{Detailed Results and Discussion}
\label{app:det_results}

\subsection{Policy Results}
\label{app:policy_results}

Table~\ref{tab:policy_results} reports the full turn-level policy results across all backbones and training setups, and Fig.~\ref{fig:policy_model_summary} provides a compact summary of the main trend. Two findings are consistent. First, pretrained instruction-tuned LLMs perform poorly on tutor-action prediction, indicating that pedagogical decision making is not recovered from general instruction-following alone. Second, action-only supervision is more effective than jointly generating diagnosis and action, and IV-calibrated training further improves both action macro-F1 and over-intervention behavior. Among all variants, Qwen-2.5-7B-Instruct with IV-calibrated action-only training performs best overall, achieving the highest action macro-F1 and the lowest over-intervention rate. The smaller Qwen-2.5-3B model is already competitive after supervised and IV-calibrated training, suggesting that the action-selection task is learnable with compact models, but the 7B policy provides the strongest overall trade-off and is therefore used in the end-to-end tutoring experiments.

\noindent\textbf{Effect of annotation source.}
We also analyze the best policy separately on \(\mathcal{D}_{\mathrm{dial}}\) and \(\mathcal{D}_{\mathrm{step}}\). The policy performs substantially better on \(\mathcal{D}_{\mathrm{step}}\) than on \(\mathcal{D}_{\mathrm{dial}}\), improving from \(0.743\) to \(0.859\) action accuracy and from \(0.426\) to \(0.651\) macro-F1. This gap suggests that \(\mathcal{D}_{\mathrm{step}}\) provides cleaner, more localized supervision, while \(\mathcal{D}_{\mathrm{dial}}\) is harder because the correct action depends on longer conversational context and teacher-led trajectory history. Intervention calibration is also sharper on \(\mathcal{D}_{\mathrm{step}}\), with over-intervention dropping from \(0.036\) to \(0.001\) and under-intervention from \(0.049\) to \(0.012\), showing that the annotation source affects both action prediction and the decision of whether intervention is necessary.

\begin{table}[htbp]
\centering
\scriptsize
\caption{\footnotesize Turn-level policy prediction results on the \textsc{DICE-Bench} test split. We report action accuracy, action macro-F1, over-intervention rate, and action macro-F1 separately on the easy variant \(q^{-}\), source problem \(q\), and hard variant \(q^{+}\). Higher is better for accuracy and F1; lower is better for over-intervention. The best value in each column is bolded and the second-best is underlined. The highlighted row indicates the best overall policy by action macro-F1. \textit{DA} means we trained the model to generate Diagnosis and then Action to be taken, \textit{A} means we generated only the Action for the given tutor state.}
\label{tab:policy_results}
\renewcommand{\arraystretch}{1.13}
\setlength{\tabcolsep}{4pt}
\begin{tabular}{llrrrrrr}
\toprule
\rowcolor{gray!25}
\textbf{Model} &
\textbf{Training setup} &
\textbf{Act. Acc.} &
\textbf{Act. F1} &
\textbf{OverInt.} \(\downarrow\) &
\textbf{\(q^{-}\) F1} &
\textbf{\(q\) F1} &
\textbf{\(q^{+}\) F1} \\
\midrule

\multirow{4}{*}{Llama-3.1-8B Ins.}
& Pretrained
& 3.42 & 3.32 & 38.42 & 2.26 & 4.05 & 3.43 \\
& SFT (\textit{DA})
& 74.54 & 57.64 & 5.82 & 60.07 & 43.50 & 60.61 \\
& SFT (\textit{A})
& 76.78 & 61.83 & 4.64 & 64.51 & 40.02 & 68.10 \\
& IV calibrated SFT (\textit{A})
& \underline{84.65} & \underline{71.43} & 2.79 & \underline{74.54} & \underline{49.75} & \textbf{79.67} \\

\midrule

\multirow{4}{*}{Qwen-2.5-3B Ins.}
& Pretrained
& 8.99 & 11.25 & 38.47 & 6.78 & 13.13 & 11.94 \\
& SFT (\textit{DA})
& 33.28 & 23.20 & 17.16 & 27.14 & 18.56 & 24.92 \\
& SFT (\textit{A})
& 72.83 & 60.22 & 10.51 & 69.41 & 35.83 & 68.71 \\
& IV calibrated SFT (\textit{A})
& 79.47 & 65.17 & 3.18 & 65.82 & 49.74 & 67.30 \\

\midrule

\multirow{4}{*}{Qwen-2.5-7B Ins.}
& Pretrained
& 14.66 & 10.15 & 26.78 & 7.62 & 12.93 & 7.32 \\
& SFT (\textit{DA})
& 57.87 & 53.70 & 7.14 & 60.10 & 29.55 & 64.95 \\
& SFT (\textit{A})
& 76.39 & 62.74 & \underline{1.76} & 63.65 & 46.81 & 62.92 \\
\rowcolor{oursblue}
& IV calibrated SFT (\textit{A})
& \textbf{84.95} & \textbf{72.13} & \textbf{1.37} & \textbf{75.08} & \textbf{52.26} & \underline{77.17} \\
\bottomrule
\end{tabular}
\end{table}

\subsection{Student Capability Analysis}
\label{app:student_capability}

A tutoring evaluation is meaningful only if the student model is neither too weak to benefit from guidance nor already strong enough to solve most problems independently. We therefore evaluate three Qwen-2.5-Instruct student backbones, 0.5B, 1.5B, and 3B, in single-turn question-answering settings before running interactive tutoring. Each model is evaluated on \(6{,}000\) problems, consisting of \(2{,}000\) easy variants \(q^{-}\), \(2{,}000\) source problems \(q\), and \(2{,}000\) hard variants \(q^{+}\). We test zero-shot answering, chain-of-thought prompting, and chain-of-thought prompting with gold subgoal steps.

\begin{table}[htbp]
\centering
\scriptsize
\caption{\footnotesize Student-only question-answering accuracy on \textsc{DICE-Bench}. Each model is evaluated on \(6{,}000\) problems, with \(2{,}000\) examples each from \(q^{-}\), \(q\), and \(q^{+}\). These results measure the intrinsic problem-solving capability of the student model without tutor interaction.}
\label{tab:student_capability}
\renewcommand{\arraystretch}{1.12}
\setlength{\tabcolsep}{5pt}
\begin{tabular}{llrrrr}
\toprule
\rowcolor{gray!25}
\textbf{Student Model} &
\textbf{Prompting Setup} &
\textbf{Overall Acc.} &
\textbf{\(q^{-}\) Acc.} &
\textbf{\(q\) Acc.} &
\textbf{\(q^{+}\) Acc.} \\
\midrule

\multirow{3}{*}{Qwen-2.5-0.5B Ins.}
& Zero-shot       & 10.25 & 19.90 & 7.70  & 3.15 \\
& CoT            & 10.92 & 20.75 & 9.15  & 2.85 \\
& CoT + Steps    & 13.35 & 23.95 & 9.10  & 7.00 \\

\midrule

\multirow{3}{*}{Qwen-2.5-1.5B Ins.}
& Zero-shot       & 20.73 & 38.25 & 17.00 & 6.95 \\
& CoT            & 21.22 & 39.30 & 17.75 & 6.60 \\
& CoT + Steps    & 24.38 & 43.20 & 18.70 & 11.25 \\

\midrule

\multirow{3}{*}{Qwen-2.5-3B Ins.}
& Zero-shot       & \textbf{32.15} & \textbf{56.85} & \textbf{30.05} & 9.55 \\
& CoT            & 30.42 & 53.70 & 28.20 & 9.35 \\
& CoT + Steps    & 30.55 & 51.50 & 26.50 & \textbf{13.65} \\

\bottomrule
\end{tabular}
\end{table}

Table~\ref{tab:student_capability} shows a clear capability gradient across student sizes. The 0.5B model solves too few problems independently, especially on \(q\) and \(q^{+}\), making it difficult to distinguish poor tutoring from an intrinsically weak student simulator. The 3B model is substantially stronger, particularly on easy and source problems, which leaves less room for tutoring to produce measurable gains. We therefore use Qwen-2.5-1.5B-Instruct as the fixed student simulator in the main experiments: it can solve a meaningful fraction of problems on its own, but still leaves substantial headroom for adaptive tutoring to improve performance, especially on harder variants.

\subsection{DICE Inference Hyperparameter Selection}
\label{app:dice_hparam_sweep}

We tune the inference hyperparameters for \textsc{DICE} on a held-out validation subset of \(150\) single-problem tutoring episodes, with \(50\) examples each from \(q^{-}\), \(q\), and \(q^{+}\). Each run treats one problem variant as one tutoring episode. We sweep the action-selection mode, teacher generator, and the corresponding inference hyperparameter: \(\lambda\) for \textsc{DICE-Cost} and \(\tau\) for \textsc{DICE-Threshold}. Model selection uses a composite validation score that rewards final-answer accuracy and penalizes unsafe or excessive tutoring behavior:
\[
\mathrm{Score}
=
\mathrm{Acc}
-
0.50\,\mathrm{FalseAcc}
-
0.50\,\mathrm{OverInt}
-
0.05\,\mathrm{Int},
\]
where \(\mathrm{FalseAcc}\) is the false-acceptance rate, \(\mathrm{OverInt}\) is the over-intervention rate, and \(\mathrm{Int}\) is the average number of active interventions per problem. We also report average successful turns as an efficiency diagnostic, but it is not used in the composite score. This score is used only for validation-time configuration selection, not as the main evaluation metric. Table~\ref{tab:dice_hparam_sweep} reports the full sweep.

\begin{table*}[htbp]
\centering
\tiny
\caption{\footnotesize Validation sweep for \textsc{DICE} inference hyperparameters on \(150\) single-problem tutoring episodes. We include \(50\) examples each from \(q^{-}\), \(q\), and \(q^{+}\). \(\lambda\) is the intervention cost for \textsc{DICE-Cost}; \(\tau\) is the intervention threshold for \textsc{DICE-Threshold}. Turns-S denotes average turns among successfully solved problems; Int./Prob. denotes average active interventions per problem. Higher is better for accuracy (0-1) and score; lower is better for turns, interventions, over-intervention (0-1), and \abv{facc} (0-1).}
\label{tab:dice_hparam_sweep}
\renewcommand{\arraystretch}{1.06}
\setlength{\tabcolsep}{2.4pt}
\begin{tabularx}{\textwidth}{
    >{\raggedleft\arraybackslash}p{0.035\textwidth}
    >{\raggedright\arraybackslash}p{0.165\textwidth}
    >{\raggedright\arraybackslash}p{0.165\textwidth}
    >{\raggedright\arraybackslash}p{0.080\textwidth}
    >{\raggedleft\arraybackslash}p{0.060\textwidth}
    >{\raggedleft\arraybackslash}p{0.070\textwidth}
    >{\raggedleft\arraybackslash}p{0.080\textwidth}
    >{\raggedleft\arraybackslash}p{0.070\textwidth}
    >{\raggedleft\arraybackslash}p{0.080\textwidth}
    >{\raggedleft\arraybackslash}p{0.060\textwidth}
}
\toprule
\rowcolor{gray!25}
\textbf{Rank} &
\textbf{Mode} &
\textbf{Teacher} &
\textbf{HP} &
\textbf{Acc.} \(\uparrow\) &
\textbf{Turns-S} \(\downarrow\) &
\textbf{Int./Prob.} \(\downarrow\) &
\textbf{OverInt.} \(\downarrow\) &
\textbf{FalseAcc.} \(\downarrow\) &
\textbf{Score} \(\uparrow\) \\
\midrule

1  & \textsc{DICE-Cost}      & Qwen-2.5-7B & \(\lambda=0.25\) & 0.865 & 1.617 & 0.885 & 0.000 & 0.277 & 0.682 \\
2  & \textsc{DICE-Threshold} & Qwen-2.5-7B  & \(\tau=0.8\)     & 0.851 & 1.595 & 1.068 & 0.034 & 0.297 & 0.632 \\
3  & \textsc{DICE-Cost}      & Qwen-2.5-7B & \(\lambda=1.0\)  & 0.815 & 1.597 & 1.192 & 0.000 & 0.281 & 0.615 \\
4  & \textsc{DICE-MAP}       & Qwen-2.5-7B & --               & 0.836 & 1.746 & 1.027 & 0.007 & 0.349 & 0.606 \\
5  & \textsc{DICE-Cost}      & SocraticLM  & \(\lambda=0.5\)  & 0.812 & 1.760 & 0.758 & 0.000 & 0.342 & 0.603 \\
6  & \textsc{DICE-Cost}      & SocraticLM  & \(\lambda=2.0\)  & 0.805 & 1.842 & 0.705 & 0.000 & 0.389 & 0.576 \\
7  & \textsc{DICE-Cost}      & SocraticLM  & \(\lambda=3.0\)  & 0.800 & 1.692 & 0.620 & 0.000 & 0.393 & 0.572 \\
8  & \textsc{DICE-Threshold} & SocraticLM & \(\tau=0.7\)     & 0.800 & 1.658 & 1.240 & 0.000 & 0.333 & 0.571 \\
9  & \textsc{DICE-Cost}      & SocraticLM  & \(\lambda=1.5\)  & 0.805 & 1.900 & 0.725 & 0.000 & 0.409 & 0.564 \\
10 & \textsc{DICE-Cost}      & Qwen-2.5-7B & \(\lambda=0.75\) & 0.795 & 1.638 & 1.164 & 0.000 & 0.349 & 0.562 \\
11 & \textsc{DICE-MAP}       & SocraticLM  & --               & 0.797 & 1.729 & 1.034 & 0.027 & 0.345 & 0.560 \\
12 & \textsc{DICE-Cost}      & Qwen-2.5-7B & \(\lambda=0.5\)  & 0.786 & 1.623 & 1.214 & 0.007 & 0.331 & 0.557 \\
13 & \textsc{DICE-Cost}      & SocraticLM  & \(\lambda=0.0\)  & 0.811 & 1.883 & 1.176 & 0.014 & 0.392 & 0.549 \\
14 & \textsc{DICE-Cost}      & SocraticLM  & \(\lambda=1.0\)  & 0.760 & 1.605 & 0.840 & 0.000 & 0.347 & 0.545 \\
15 & \textsc{DICE-Cost}      & Qwen-2.5-7B & \(\lambda=2.0\)  & 0.785 & 1.752 & 1.396 & 0.000 & 0.354 & 0.538 \\
16 & \textsc{DICE-Cost}      & SocraticLM  & \(\lambda=0.25\) & 0.793 & 1.849 & 1.047 & 0.013 & 0.400 & 0.534 \\
17 & \textsc{DICE-Threshold} & SocraticLM  & \(\tau=0.8\)     & 0.767 & 1.591 & 0.993 & 0.007 & 0.373 & 0.527 \\
18 & \textsc{DICE-Threshold} & Qwen-2.5-7B & \(\tau=0.7\)     & 0.812 & 2.000 & 1.423 & 0.047 & 0.389 & 0.523 \\
19 & \textsc{DICE-Cost}      & Qwen-2.5-7B & \(\lambda=1.5\)  & 0.761 & 1.593 & 1.528 & 0.000 & 0.338 & 0.515 \\
20 & \textsc{DICE-Cost}      & Qwen-2.5-7B & \(\lambda=3.0\)  & 0.738 & 1.382 & 1.295 & 0.000 & 0.322 & 0.512 \\
21 & \textsc{DICE-Cost}      & Qwen-2.5-7B & \(\lambda=0.0\)  & 0.764 & 1.566 & 1.392 & 0.020 & 0.345 & 0.511 \\
22 & \textsc{DICE-Cost}      & SocraticLM  & \(\lambda=0.75\) & 0.758 & 1.673 & 0.893 & 0.000 & 0.416 & 0.506 \\
23 & \textsc{DICE-Threshold} & SocraticLM  & \(\tau=0.6\)     & 0.743 & 2.055 & 2.568 & 0.189 & 0.331 & 0.355 \\
24 & \textsc{DICE-Threshold} & Qwen-2.5-7B & \(\tau=0.6\)     & 0.707 & 1.915 & 2.533 & 0.233 & 0.353 & 0.287 \\
25 & \textsc{DICE-Threshold} & SocraticLM  & \(\tau=0.4\)     & 0.140 & 3.095 & 7.313 & 0.753 & 0.373 & -0.789 \\
26 & \textsc{DICE-Threshold} & SocraticLM  & \(\tau=0.3\)     & 0.153 & 3.609 & 7.327 & 0.747 & 0.427 & -0.800 \\
27 & \textsc{DICE-Threshold} & Qwen-2.5-7B & \(\tau=0.5\)     & 0.120 & 3.889 & 7.500 & 0.707 & 0.427 & -0.822 \\
28 & \textsc{DICE-Threshold} & Qwen-2.5-7B & \(\tau=0.3\)     & 0.113 & 3.824 & 7.527 & 0.740 & 0.393 & -0.830 \\
29 & \textsc{DICE-Threshold} & Qwen-2.5-7B & \(\tau=0.4\)     & 0.087 & 3.385 & 7.600 & 0.700 & 0.393 & -0.840 \\
30 & \textsc{DICE-Threshold} & SocraticLM  & \(\tau=0.5\)     & 0.100 & 3.733 & 7.567 & 0.740 & 0.493 & -0.895 \\

\bottomrule
\end{tabularx}
\end{table*}

The sweep shows that cost-sensitive inference is more stable than thresholding across the tested range. The best configuration is \textsc{DICE-Cost} with Qwen-2.5-7B-Instruct and \(\lambda=0.25\), which obtains the highest validation score by combining high accuracy with low turn count and zero over-intervention. Threshold-based inference can also perform well when \(\tau\) is large, but low thresholds cause the tutor to intervene too often, increasing turns and sharply reducing accuracy. Based on this sweep, we select the strongest cost-sensitive, threshold, and MAP configurations for the larger single-problem and session-level evaluations.

\subsection{\textsc{DICE} Acceptance Verification and Final-Answer Guard}
\label{app:verifier_guard_discussion}

\noindent\textbf{Making accept decisions explicit.}
Every tutoring system must decide when a student's response should be accepted. Response-only tutors make this decision implicitly through generated language, for example by saying ``Correct,'' ``Good,'' or ``Now move on.'' In contrast, \textsc{DICE} makes this decision explicit through the non-intervention action \(a_0=\mathrm{accept\_and\_advance}\). This explicit action interface is important: once acceptance is represented as a discrete decision, it can be checked before the session advances.

\noindent\textbf{Why verification is needed.}
The teacher used in our experiments is not an oracle, a mathematical judge, or a frontier-scale verifier. It is a standard instruction-tuned model used to simulate a realistic tutoring system. Such models can generate useful Socratic guidance, but they cannot always reliably decide whether a student's latest response is mathematically correct, especially when the response is long, partially correct, superficially fluent, or phrased as a final answer. This matters in session-level tutoring because an incorrect accept decision can prematurely move the student from \(q^{-}\) to \(q\), or from \(q\) to \(q^{+}\), under a shared turn budget.

\noindent\textbf{Acceptance verifier.}
For session-level \textsc{DICE}, the verifier is invoked only when the policy selects \(a_0=\mathrm{accept\_and\_advance}\). The verifier receives the problem, reference answer, solution context, dialogue history, and latest student response, and predicts whether the response is \(\mathrm{CORRECT}\), \(\mathrm{INCORRECT}\), or \(\mathrm{UNCLEAR}\). If the response is not verified as \(\mathrm{CORRECT}\), the accept decision is blocked and replaced with \(\mathrm{ask\_for\_justification}\). Thus, the verifier does not control every tutor action; it audits the single action whose failure mode is most consequential: advancing after an incorrect response.

\noindent\textbf{Final-answer guard.}
We further define unsafe final-answer acceptance as a stricter subset of \abv{facc}: the tutor accepts an incorrect response that appears to be the student's final answer. This is more severe than accepting an incomplete intermediate step because it validates the student's final conclusion. The final-answer guard therefore applies a conservative rule: if the latest response appears final-answer-like but fails exact answer matching, it is treated as \(\mathrm{UNCLEAR}\), even when the verifier is overly permissive. This prevents the tutor from validating wrong final answers.

\noindent\textbf{Interpreting the baseline trade-off.}
Unsafe final-answer acceptance should not be interpreted in isolation. A response-only tutor can obtain a low unsafe-acceptance rate by rarely accepting at all, continuing to question the student even when the student is already correct. This explains why Socratic prompting and MathDial can show relatively low unsafe acceptance while also exhibiting high over-intervention and high turn usage. SocraticLM shows the opposite trade-off: it acknowledges progress more readily, which reduces turns, but increases the risk of accepting incorrect responses. These baselines expose two competing tutoring failures: over-intervening when the student is correct and falsely accepting when the student is wrong.

\noindent\textbf{Role in \textsc{DICE}.}
The verifier and guard are therefore not post-hoc corrections to evaluation labels. They are part of the \textsc{DICE} execution layer enabled by explicit action selection. The policy reduces unnecessary intervention by distinguishing \(a_0\) from active actions \(\mathcal{A}^{+}\), while the verifier and guard make accept decisions safer when \(a_0\) is selected. This supports the central goal of \textsc{DICE}: not merely to speak less, but to intervene selectively while avoiding premature validation of incorrect student answers.

\section{Additional Qualitative Examples}
\label{sec:addon_qual}

\subsection{Qualitative Dialogue Rollouts}
\label{app:qual_dialogue_rollouts}

Table~\ref{tab:qual_dialogue_rollouts} shows two representative dialogue rollouts comparing \textsc{SocraticLM} and \textsc{DICE}. The first case illustrates repeated over-intervention after the student has reached the correct solution, while the second shows false acceptance, where the baseline validates an incorrect answer and continues along the wrong trajectory. In both cases, \textsc{DICE} makes an explicit action decision before generation, enabling targeted correction or acceptance.

\begin{table}[htbp]
\centering
\scriptsize
\caption{\footnotesize Qualitative examples of single-problem dialogue rollouts comparing \textsc{SocraticLM} and \textsc{DICE}. To capture realistic simulator variability, student generations were not seeded, while the student model was kept fixed across systems. For readability, we remove repetitive/noisy turns and retain only the essential exchanges needed to illustrate tutor behavior. The first example shows repeated over-intervention, while the second shows false acceptance of an incorrect student answer. Colored tags indicate the tutor action or failure mode.}
\label{tab:qual_dialogue_rollouts}
\renewcommand{\arraystretch}{1.18}
\setlength{\tabcolsep}{4pt}
\begin{tabularx}{\textwidth}{p{0.15\textwidth}p{0.12\textwidth}Yp{0.18\textwidth}}
\toprule
\rowcolor{gray!25}
\textbf{Case} & \textbf{Tutor} & \textbf{Problem and dialogue rollout} & \textbf{Behavior} \\
\midrule

\multirow{2}{=}{\textbf{Over-intervention}
\vspace{1mm}

\textbf{\textit{Problem}:} Mary earns \(\$46\) for cleaning one home. She earned \(\$300\). How many whole homes could she clean and how much money would be left over?

\textbf{\textit{Answer}:} \(6\) homes, \(\$24\) left.}
&
\actionbadge{modelbg}{modelfg}{SocraticLM}
&
\textbf{S:} \(300\div46=6.52\), so Mary cleaned \(6\) homes. The leftover is \(300-(46\times6)=\$8\).

\textbf{T:} \actionbadge{badbg}{badfg}{False acceptance} That's correct! Mary cleaned \(6\) whole homes and had \(\$8\) left over.

--- \textit{Student acknowledges, dialogue continues for 5 turns}.

\textbf{T:} There seems to be a small error. Let's try again: \(300-(46\times6)=\$84-\$276=\$124\).

\textbf{S:} Mary cleaned \(6\) homes and had \(\$124\) left.

\textbf{T:} Recalculate the remainder: \(300-\$276=\$24\).

\textbf{S:} Mary cleaned \(6\) homes and had \(\$24\) left.

\textbf{T:} \actionbadge{warnbg}{warnfg}{Over-intervention} Let's try it again: \(300-\$276=\$24\). 

--- \textit{{Student gets confused and misguided, and teacher continues intervention}}
&
The baseline first accepts an incorrect remainder, then repeatedly intervenes even after the student reaches the correct answer. It continues intervening till set budget. \\

\cmidrule(lr){2-4}

&
\actionbadge{modelbg}{modelfg}{DICE}
&
\textbf{S:} Mary could clean \(7\) whole homes since \(46\times7=322\), and have \(\$18\) left over.

\textbf{T:} \actionbadge{repairbg}{repairfg}{Arithmetic repair} Let's recheck the division of \(\$300\) by \(\$46\). How many whole homes can Mary clean with \(\$300\)?

\textbf{S:} Mary can clean \(6\) whole homes, since \(46\times6=\$276\).

\textbf{T:} \actionbadge{acceptbg}{acceptfg}{Accept and advance} That's correct. Please continue to the next step.

\textbf{S:} The leftover is \(300-276=\$24\).

\textbf{T:} \actionbadge{acceptbg}{acceptfg}{Accept and advance} That's correct.
&
\textsc{DICE} applies one targeted arithmetic repair, then accepts once the student is on track. \\

\midrule

\multirow{2}{=}{\textbf{False acceptance}
\vspace{1mm}

\textbf{\textit{Problem}:} A library bookcase held \(825\) books. On Monday \(247\) were borrowed; on Tuesday \(198\) were returned; on Wednesday \(156\) more were borrowed. How many books are on the shelf at the end of Wednesday?

\textbf{\textit{Answer}:} \(620\).}
&
\actionbadge{modelbg}{modelfg}{SocraticLM}
&
\textbf{S:} \(825-247+198-156=420\). So there are \(420\) books remaining.

\textbf{T:} \actionbadge{badbg}{badfg}{False acceptance} That's correct! You've successfully calculated the number of books left on the shelf at the end of Wednesday. Well done!

\textbf{S:} Thank you for your encouragement! It's great to see I was able to solve this problem correctly.

\textbf{T:} Keep up the good work!
&
The baseline validates the wrong answer \(420\), reinforcing an incorrect trajectory instead of repairing the arithmetic. \\

\cmidrule(lr){2-4}

&
\actionbadge{modelbg}{modelfg}{DICE}
&
\textbf{S:} \(825-247+198-156=420\). So there are \(420\) books remaining.

\textbf{T:} \actionbadge{repairbg}{repairfg}{Arithmetic repair} Let's recheck the subtraction and addition step by step. You started with \(825\) books. After Monday, \(247\) were borrowed, so you did \(825-247\). Can you verify that calculation?

\textbf{S:} \(825-247=578\). Then \(578+198=776\), and \(776-156=620\).

\textbf{T:} \actionbadge{acceptbg}{acceptfg}{Accept and advance} That's correct.
&
\textsc{DICE} detects the arithmetic error, repairs the intermediate step, and accepts only after the student reaches the correct answer. \\

\bottomrule
\end{tabularx}
\end{table}

\section{Limitations and Future Work}
\label{sec:limitations}

\noindent\textbf{Limitations.}
This work evaluates intervention-calibrated tutoring through controlled LLM--LLM simulation rather than human-subject studies, so our results should be interpreted as evidence of improved simulated tutoring behavior, not direct evidence of long-term human learning gains such as retention, transfer, or student confidence. However, this scope is deliberate: estimating Intervention Value requires counterfactual comparison between intervention and non-intervention at the same tutoring state, which is difficult to obtain at scale with human learners and high learner variability. Simulation lets us compare policies on identical states, budgets, and student models, and enables dense per-decision analysis across \textsc{DICE-Bench}. Our benchmark is also limited to mathematical word problems with skill-preserving difficulty variants; while the intervention question is general, our current implementation relies on math-specific subgoal decomposition, reference answers, and exact-answer verification. Finally, IV estimation uses model-based rollouts and an LLM judge, which may inherit simulator and evaluator biases, although the controlled comparison partially mitigates this by evaluating intervention and non-intervention arms under the same simulator, judge, and rubric.

\noindent\textbf{Future work.}
Future work should extend \textsc{DICE} beyond math toward general reasoning tutors for programming, scientific reasoning, logic, and commonsense problem solving, where progress and acceptance may require domain-specific verification such as test execution, derivation checking, or factual validation. Another direction is continual and adaptive tutoring: online IV refinement could update intervention values from new tutor--student interactions, allowing policies to adapt to different student populations, curricula, and error patterns without full retraining. Finally, the controlled simulation framework developed here provides a tractable foundation for human-subject studies, where \textsc{DICE}'s explicit action interface can support inspection, auditing, and adjustment of intervention decisions before measuring retention, transfer, and student confidence in real learning settings.

\section{Broader Impact}
\label{apdx:broader_impact}

\textsc{DICE} aims to make automated tutoring more intervention-aware, which could improve the scalability and efficiency of personalized educational support. By explicitly modeling when a tutor should intervene, the framework may help reduce unnecessary feedback, preserve student agency, and make limited tutoring budgets more effective. At the same time, LLM tutors can make incorrect pedagogical decisions, over-trust student answers, or provide misleading guidance, especially when evaluated only with simulated students. These risks are particularly important in educational settings, where incorrect feedback can reinforce misconceptions. Deployment with real learners should include human oversight, careful domain validation, and monitoring for unequal performance across student groups.


\section{List of Abbreviations}
\label{app:abbreviations}

\begin{center}
\scriptsize
\begin{tabular}{@{}p{0.18\textwidth}p{0.74\textwidth}@{}}
\toprule
\textbf{Notation} & \textbf{Description} \\
\midrule
\textbf{CLT} & Cognitive Load Theory \\
\textbf{CoT} & Chain-of-Thought \\
\textbf{DAG} & Directed Acyclic Graph \\
\textbf{FA} & False Acceptance rate \\
\textbf{IV} & Intervention Value \\
\textbf{KL} & Kullback--Leibler \\
\textbf{LLM} & Large Language Model \\
\textbf{OInt.} & Over Intervention rate \\
\textbf{SFT} & Supervised Fine-tuning \\
\bottomrule
\end{tabular}
\end{center}



\end{document}